\documentclass[letterpaper]{article} 
\usepackage{aaai2026}  
\usepackage{times}  
\usepackage{helvet}  
\usepackage{courier}  
\usepackage[hyphens]{url}  
\usepackage{graphicx} 
\usepackage{natbib}  
\usepackage{caption} 
\usepackage[ruled,vlined,linesnumbered]{algorithm2e}

\usepackage{xfrac}
\usepackage{adjustbox}
\usepackage[caption=false]{subfig}
\usepackage{placeins}
\usepackage{multirow}
\usepackage{xcolor}
\usepackage{float}
\usepackage{amssymb}
\usepackage{amsmath}
\usepackage{amsthm}
\usepackage{booktabs}

\newtheorem{theorem}{Theorem}
\newtheorem{lemma}[theorem]{Lemma}

\usepackage{newfloat}
\usepackage{listings}
\DeclareCaptionStyle{ruled}{labelfont=normalfont,labelsep=colon,strut=off} 
\floatstyle{ruled}
\newfloat{listing}{tb}{lst}{}
\floatname{listing}{Listing}
\title{Bidirectional Bounded-Suboptimal Heuristic Search with Consistent Heuristics}

\newif\ifcomments
\commentstrue 

\definecolor{shahafcolor}{RGB}{255, 0, 0} 
\definecolor{dorcolor}{RGB}{53, 100, 0} 
\definecolor{arielcolor}{RGB}{0, 0, 255} 
\definecolor{nataliecolor}{RGB}{255, 165, 0} 
\definecolor{liorcolor}{RGB}{128, 0, 128} 

\ifcomments
    \newcommand{\shahaf}[1]{\textcolor{shahafcolor}{[Shahaf: #1]}}
    \newcommand{\dor}[1]{\textcolor{dorcolor}{[Dor: #1]}}
    \newcommand{\ariel}[1]{\textcolor{arielcolor}{[Ariel: #1]}}
    \newcommand{\natalie}[1]{\textcolor{nataliecolor}{[Natalie: #1]}}
    \newcommand{\lior}[1]{\textcolor{liorcolor}{[Lior: #1]}}
\else
    \newcommand{\shahaf}[1]{}
    \newcommand{\dor}[1]{}
    \newcommand{\ariel}[1]{}
    \newcommand{\natalie}[1]{}
    \newcommand{\lior}[1]{}
\fi

\author{
    Shahaf S. Shperberg\textsuperscript{\rm 1}, 
    Natalie Morad\textsuperscript{\rm 1},
    Lior Siag\textsuperscript{\rm 1},
    Ariel Felner\textsuperscript{\rm 1}, 
    Dor Atzmon\textsuperscript{\rm 2}
}
\affiliations{
    
    \textsuperscript{\rm 1} Faculty of Computer and Information Science, Ben-Gurion University of the Negev, Israel	\\
    \textsuperscript{\rm 2} Department of Computer Science, Bar-Ilan University, Israel	\\

    \{shperbsh,felner\}@bgu.ac.il,   \{moradna,siagl\}@post.bgu.ac.il, dor.atzmon@biu.ac.il 
}

\usepackage{xspace}
\usepackage{accents}
\newcommand\thickbar[1]{\accentset{\rule{.4em}{.6pt}}{#1}}

\newcommand{\openf}{\textsc{Open\textsubscript{F}}\xspace}      
\newcommand{\openb}{\textsc{Open\textsubscript{B}}\xspace}      
\newcommand{\gminf}{\mbox{$gMin_F$}\xspace}     
\newcommand{\fmin}{\mbox{$fMin$}\xspace}        
\newcommand{\fminf}{\mbox{$fMin_F$}\xspace}     
\newcommand{\fminb}{\mbox{$fMin_B$}\xspace}     
\newcommand{\fmind}{\mbox{$fMin_D$}\xspace}     

\newcommand{\prW}{\mbox{$pr_W$}\xspace}
\newcommand{\prWD}{\mbox{$pr_{W_D}$}\xspace}
\newcommand{\prWF}{\mbox{$pr_{W_F}$}\xspace}

\newcommand{\prWminf}{\mbox{$pr_WMin_F$}\xspace}
\newcommand{\prWminb}{\mbox{$pr_WMin_B$}\xspace} 
\newcommand{\prWmind}{\mbox{$pr_WMin_D$}\xspace} 

\newcommand{\unihs}{UHS\xspace}              
\newcommand{\bihs}{BiHS\xspace}                
  
\newcommand{\astar}{A*\xspace}                  
\newcommand{\mm}{MM\xspace}                     

\newcommand{\wastar}{WA*\xspace}                
\newcommand{\wmm}{WMM\xspace}                   
\newcommand{\bwa}{WBiA$^*$\xspace}   
   
\newcommand{\WBSstar}{WBS$^*$\xspace}

\newcommand{\bae}{\mbox{BAE$^*$}\xspace}
\newcommand{\BAE}{\mbox{BAE$^*$}\xspace}
\newcommand{\bWae}{\mbox{WBAE$^*$}\xspace}

\newcommand{\BSstar}{\mbox{BS$^*$}\xspace}

\newcommand{\Cstar}{\ensuremath{C^*}}

\newcommand{\openF}{\textsc{Open\textsubscript{F}}\xspace}      
\newcommand{\openB}{\textsc{Open\textsubscript{B}}\xspace}      
\newcommand{\openD}{\textsc{Open\textsubscript{D}}\xspace}      
\newcommand{\opend}{\textsc{Open\textsubscript{D}}\xspace}      

\newcommand{\start}{\ensuremath{start}\xspace}
\newcommand{\goal}{\ensuremath{goal}\xspace}

\newcommand{\BiHS}{\mbox{BiHS}\xspace}     
\newcommand{\UniHS}{\mbox{UniHS}\xspace}   

\newcommand{\bss}{BSS\xspace}              

\newcommand{\MM}{\mbox{MM}\xspace}

\newcommand{\gD}{\mbox{$g_D$}}
\newcommand{\gF}{\mbox{$g_F$}}
\newcommand{\gB}{\mbox{$g_B$}}

\newcommand{\hD}{\mbox{$h_D$}}
\newcommand{\hF}{\mbox{$h_F$}}
\newcommand{\hB}{\mbox{$h_B$}}
\newcommand{\hO}{\mbox{$h_{\thickbar{D}}$}}

\newcommand{\fW}{\mbox{$f_W$}}

\newcommand{\fD}{\mbox{$f_D$}}
\newcommand{\fF}{\mbox{$f_F$}}

\newcommand{\prmin}{\mbox{$PrMin$}\xspace}

\newcommand{\dD}{\mbox{$d_D$}}
\newcommand{\dF}{\mbox{$d_F$}}
\newcommand{\dB}{\mbox{$d_B$}}

\newcommand{\bD}{\mbox{$b_D$}}
\newcommand{\bF}{\mbox{$b_F$}}
\newcommand{\bB}{\mbox{$b_B$}}

\newcommand{\bWD}{\mbox{$b_{W_D}$}}
\newcommand{\bWF}{\mbox{$b_{W_F}$}}
\newcommand{\bWB}{\mbox{$b_{W_B}$}}

\newcommand{\alb}{\mbox{ALB}\xspace}

\newcommand{\bminf}{\mbox{$bMin_F$}\xspace}     
\newcommand{\bminb}{\mbox{$bMin_B$}\xspace}     
\newcommand{\bmind}{\mbox{$bMin_D$}\xspace}

\newcommand{\bWminf}{\mbox{$b_WMin_F$}\xspace}     
\newcommand{\bWminb}{\mbox{$b_WMin_B$}\xspace}     
\newcommand{\bWmind}{\mbox{$b_WMin_D$}\xspace}

\newcommand{\defeq}{=}

\begin{document}

\maketitle

\begin{abstract}
Recent advancements in bidirectional heuristic search have yielded significant theoretical insights and novel algorithms.
While most previous work has concentrated on optimal search methods, this paper focuses on bounded-suboptimal bidirectional search, where a bound on the suboptimality of the solution cost is specified.
We build upon the state-of-the-art optimal bidirectional search algorithm, BAE*, designed for consistent heuristics, and introduce several variants of BAE* specifically tailored for the bounded-suboptimal context.
Through experimental evaluation, we compare the performance of these new variants against other bounded-suboptimal bidirectional algorithms as well as the standard weighted A* algorithm.
Our results demonstrate that each algorithm excels under distinct conditions, highlighting the strengths and weaknesses of each approach.
\end{abstract}

\begin{links}
    \link{Appendix \& Code}{https://github.com/SPL-BGU/Bounded-Suboptimal-BAE}
\end{links}

\section{Introduction}
In optimal search, the task is to find the least-cost path between two vertices, $start$ and $goal$, in a given graph.
Admissible \emph{Unidirectional Heuristic Search} (\unihs) algorithms like \astar~\citep{hartNR68Astar}, which prioritize nodes by $f(n) = g(n) + h(n)$, return an optimal solution~\citep{dechter} (of cost $C^*$) if $h$ is admissible (never overestimating).

In many cases, finding optimal solutions is infeasible due to the immense computation required.
\emph{Bounded-suboptimal search} (\bss) is a paradigm that trades the quality of the solution for faster running time.
In \bss, we are given a bound $W \geq 1$, and are required to find a solution with cost $ \leq W \cdot C^*$.
A classical \bss algorithm is \emph{Weighted A*} (\wastar)~ \citep{WASTR70}, which expands nodes $n$ according to $f(n)=g(n)+W\cdot h(n)$.
Given an admissible heuristic, \wastar is guaranteed to return a bounded-suboptimal solution.
Other \bss algorithms include \emph{Dynamic Potential Search} (DPS)~\citep{gilon2016dynamic}, Explicit Estimation Search (EES)~\citep{Thayer2011BoundedSS}, XDP and XUP~\citep{Chen19,Chen2021NecessaryAS}, and improved versions of these algorithms~\citep{Fickert0R22}.

\emph{Bidirectional heuristic search} (\bihs) is an alternative to \unihs that progresses simultaneously from $start$ (forward) and $goal$ (backward) until the two frontiers meet.
Recent line of work presented theoretical study and practical ways for using \bihs with tremendous achievements~\citep{barker2015limitations,MMAAAI16,eckerle2017sufficient,SahamFCS17,Chen2017FronttoEndBH,shperberg2019improving,shperberg2019enriching,ShperbergDFS21,alcazar20unifying,acazar21,siag2023comparing}.
Yet, while unidirectional \bss was broadly studied, bidirectional \bss has received limited attention.

\BAE~\citep{sadhukhan2013bidirectional} (and the identical DIBBS algorithm~\citep{SewellJ21}), is a \bihs optimal algorithm designed for consistent heuristics.
\BAE prioritizes nodes in the open lists according to $b(n)=g(n)+h(n)+d(n)$ with $d(n)$ representing a lower bound on the heuristic error of paths discovered on the opposite search that passes through $n$.
\bae is considered the state-of-the-art optimal \bihs algorithm, with several studies highlighting its superiority~\citep{alcazar20unifying,siag2023comparing}.

In this work, we introduce a family of algorithms, called \bWae, that adapt \BAE into the \bss
paradigm by incorporating the basic idea of \wastar, inflating $h$ by $W$.
Members of this family differ in how they treat $d$.
We prove that multiplying $d$ by any real number $\lambda \leq W$ results in a \bss solution.
Additionally, we discuss two methods that achieve tighter lower bounds for \bWae (as well as for other algorithms) and allow earlier termination, albeit sometimes at the cost of additional time overhead.
We analyze the impact of $\lambda$ and the tighter bounds, offering recommendations on how to set $\lambda$ and choose the most effective termination condition.

\section{Bidirectional Heuristic Search: Definitions}

In \bihs, the aim is to find a least-cost path, of cost \Cstar, between \start and \goal in a given graph $G$.
$\mathit{c}(x,y)$ denotes the cost of the cheapest path between $x$ and $y$, so $\mathit{c}(start, goal)= \Cstar$.
\bihs executes a forward search (F) from \start and a backward search (B) from \goal until the two searches meet.
\bihs\ algorithms typically maintain two open lists \openf and \openb for the forward and backward searches, respectively.
Each node has a $g$-value, an $h$-value, and an $f$-value ($g_F, h_F, f_F$ and $g_B, h_B, f_B$ for the forward and backward searches, correspondingly).
For a direction $D$ (F or B), \fD, \gD, and \hD\ represent the $f$-, $g$-, and $h$-values in that direction. We use \mbox{$xMin_D$} to denote the minimal $x$ value in \openD, e.g., \gminf represents the minimal $g$-value in \openf.
Most \bihs algorithms consider the two \emph{front-to-end} heuristic functions~\citep{Kaindl1997} $\hF(s)$ and $\hB(s)$ which respectively estimate $\mathit{c}(s,goal)$ and $\mathit{c}(\start,s)$ for all $s \in G$.
$h_F$ is \emph{forward admissible}\ iff $h_F(s)\le \mathit{c}(s,goal)$ for all $s$ in $G$ and is \emph{forward consistent}\ iff
$h_F(s)\le \mathit{c}(s,s^\prime)+h_F(s^\prime)$ for all $s$ and $s^\prime$ in~$G$.
Backward \emph{admissibility} and \emph{consistency} are defined analogously.

\begin{algorithm}[tb]
\caption{\BiHS General Framework}
\DontPrintSemicolon
\label{alg:framework}
$U \leftarrow \infty$; 
$LB \leftarrow \mathit{ComputeLowerBound()}$;\\
\While{${\normalfont \openF} \neq \emptyset \wedge {\normalfont \openB} \neq \emptyset \wedge U > LB$}{
    $D \leftarrow \mathit{ChooseDirection()}$;\\
    $n \leftarrow \mathit{ChooseNode}(D)$;\\
    $\mathit{Expand}(n, D)$; \tcp*{Also updates U}
    $LB \leftarrow \mathit{ComputeLowerBound()}$
}
\Return $U$
\end{algorithm}

\bihs algorithms differ primarily in their direction and node selection strategies, and in their termination criteria.
A general \bihs framework is presented in Algorithm~\ref{alg:framework}.
$U$ represents the cost of the incumbent solution and is returned if $U \leq LB$, where $LB$ is a lower bound on the cost of the desired solution (either optimal or bounded suboptimal).
During each expansion cycle, the algorithm chooses a direction $D$, then selects a node to expand from that direction.
New nodes are matched against the opposite frontier, and if a better solution is found, $U$ is updated.
Subsequently, $LB$ is updated as the content of \openD has changed.
\bihs algorithms vary in how they choose the direction, select nodes, and compute $LB$.
A common direction choosing policy is the \emph{alternate} policy, which chooses to expand from \openf and \openb in a round-robin fashion.
Unless stated otherwise, all algorithms in this paper use this policy.

\section{Background: \bihs Algorithms}

We review several optimal and suboptimal \bihs algorithms. 

\subsection{\wastar and Bidirectional \wastar}

A key \bss algorithm is weighted \astar (\wastar)~\citep{WASTR70}, which replaces \astar's $f(n) = g(n) + h(n)$~\citep{hartNR68Astar} priority function with $\fW(n) = g(n) + W \cdot h(n)$, inflating the heuristic by a user-defined $W \geq 1$. This guarantees a solution with cost of at most $W \cdot \Cstar$. 


BHPA~\citep{Pohl71}, or Bidirectional \astar (BiA$^*$), directly generalizes \astar to a \bihs algorithm.
BiA$^*$ orders nodes in \openD according to $\fD(n)=\gD(n)+\hD(n)$.
Perhaps the most straightforward \bss algorithm, introduced by \citet{KOLL93}, extends the \wastar approach to both search frontiers, effectively generalizing BiA$^*$ to \bihs, a variant denoted here as Bidirectional \wastar (\bwa).
\bwa prioritizes nodes in \openD according to:
\begin{equation*}
\prWD(n) \defeq g_D(n)+W\cdot h_D(n)
\end{equation*}

\bwa terminates once the same state $n$ is found on both open lists and the cost of the path from \start to \goal through $n$ is $\leq LB_{W}$, where 
\begin{equation} \label{eq:LB-BWA}
    LB_{W} \defeq \max(\prWminf , \prWminb)
\end{equation}
In this definition, \prWminf and \prWminb are the minimal $\prW$-values in \openf and \openb, respectively. \citet{KOLL93} did not provide proof for its suboptimality. For completeness we provide it in Appendix A.

\subsection{Meet in the middle (\mm) and \wmm}

The \emph{Meet in the middle} (\mm) algorithm~\citep{MMAAAI16} is a notable \bihs that ensures that the search frontiers \emph{meet in the middle} (i.e., no node $n$ is expanded with $g(n)>\Cstar/2$).
In \mm, nodes $n$ in \openD are prioritized by: 
\begin{equation*}
    pr_D(n) \defeq \max(f_D(n),2g_D(n))
\end{equation*}
\mm expands the node with minimal priority, $\prmin$, among all nodes in \openf and \openb.

\paragraph{\bf Weighted \MM.} Recently, \mm was generalized to a \bss version called \wmm~\citep{WMM}.
The main idea is to follow the WA* approach and inflate the heuristic by $W$.
Several variants were proposed, and the best one prioritizes nodes in \openD using the following formula:
\begin{equation*}
    pr_D(n)  \defeq g_D(n)+\max{(g_D(n),W\cdot h_D(n))}
\end{equation*}

\wmm is {\em W-restrained}, i.e.
the forward search never expands a node $n$ with $g_F(n) > \sfrac{W}{2}  \cdot C^*$ and the backward search never expands a node $n$ with $ g_B(n) > \sfrac{W}{2} \cdot C^*$.

\subsection{BS* and its \bss variant}

All previously discussed algorithms assume an \emph{admissible} heuristic. In contrast, \BSstar~\citep{Kwa89} also requires \emph{consistency} and improves upon BiA$^*$ by leveraging this property for tighter lower bounds and pruning via “nipping” and “trimming.” It uses Pohl's {\em cardinality criterion}~\citep{Pohl71}, expanding from the side with fewer open nodes.
\citet{KOLL93} proposed \WBSstar, a bounded-suboptimal variant of \BSstar. Like \bwa, it uses $f_W(n) \defeq g(n)+W\cdot h(n)$ and applies \BSstar's enhancements to $f_W$ instead of $f$.

\subsection{BAE*}

\BAE~\citep{sadhukhan2013bidirectional} and the equivalent, DIBBS~\citep{SewellJ21}, are more recent \bihs algorithms that, like \BSstar, assume heuristic consistency.
However, \BAE and DIBBS exploit this consistency more effectively than \BSstar, leading to improved search performance.

Let $\dF(n) \defeq \gF(n) - \hB(n)$, the {\em difference} between the actual forward cost of $n$ (from $\start$) and its heuristic estimation to \start.
This indicates the {\em heuristic error} for node $n$ (as $\hB(n)$ is a possibly inaccurate estimation of $\gF(n)$).
Likewise,  $\dB(m) \defeq \gB(m) - \hF(m)$.
\BAE orders nodes in \openD according to: 
\begin{equation*}
    \bD(n) \defeq \gD(n) +\hD(n) +(\gD(n) - \hO(n)) = \fD(n) + \dD(n)
\end{equation*}

\noindent where $\hO$ is the heuristic in the direction opposite of $D$.
$\bD(n)$ adds the heuristic error $\dD(n)$ to $\fD(n)$ to account for the underestimation by $\hO(n)$.
During each expansion cycle, \BAE chooses a search direction $D$ and expands a node with minimal $\bD$-value.
Additionally, \BAE terminates once the same state $n$ is found on both open lists and the cost of the path from \start to \goal through $n$ is $\leq LB_B$, where $LB_B$ is the following lower bound on $\Cstar$:
\begin{equation} \label{eq:LB-B}
    LB_{B} \defeq {(\bminf + \bminb)}/{2}
\end{equation}

\noindent in which \bmind is the minimal $b$-value in \opend.

Given a consistent heuristic, \bae was proven to return an optimal solution.
$b(n)$ is more informed than other priority functions (as it also considers $d(n)$). \bae was shown to outperform common unidirectional and bidirectional algorithms on different domains~\citep{alcazar20unifying,siag2023front} with improvements reported of up to an order of magnitude.
Therefore, \bae is considered a state-of-the-art \bihs algorithm for consistent heuristics.
In this paper, we develop and study \bss variants of \bae.

Other \bss \bihs algorithms have been proposed. A*-connect~\citep{astarconnect}, designed for motion planning, extends BiA$^*$ using an additional inadmissible heuristic to quickly connect the frontiers. As our approach does not rely on such heuristics, A*-connect is not a direct competitor. Similarly, R2R~\citep{rice2012bidirectional} is an \emph{additive} \bss algorithm with guarantees of $\Cstar + W$, whereas we focus on \emph{multiplicative} bounds ($W \cdot \Cstar$), making R2R incomparable in scope.

\section{Weighted BAE*}

We now introduce \bWae, a \bss\ variant of \bae that integrates the \wastar idea of inflating $h$ by a factor \(W \geq 1\) to ensure a solution within \(W \cdot \Cstar\). 

Recall the \bae priority function:
\begin{equation*}
    \bD(n) = \gD(n) +\hD(n) +\dD(n)
\end{equation*}

In \wastar, the $h$-value is scaled by \(W\) to guide the search while preserving suboptimality bounds---a strategy we retain. The key challenge is handling the heuristic error \(\dD(n)\), which has not been addressed in \bss\ settings.

We define the \bWae priority function as:
\begin{equation} \label{eq:BAE}
    \bWD(n) \defeq \gD(n) + W \cdot \hD(n) + \lambda \cdot \dD(n)
\end{equation}

Here, \(\lambda \leq W\) controls the impact of the error term. Special cases include: \(\lambda = 0\), yielding a \bwa-like function; \(\lambda = W\), fully incorporating \(\dF(n)\) into the heuristic; and intermediate values, which balance between the two.

\bWae terminates once a state $n$ is found on both open lists and the cost of the path from \start to \goal through $n$ is $\leq LB_{\mathtt{WB}}$, where:
\begin{equation} \label{eq:LB-BW}
    LB_{\mathtt{WB}} = {(\bWminf + \bWminb)}/{2}
\end{equation}

\noindent and \bWmind is the minimal $\bWD$-value in \opend.

Notably, BAE* guarantees optimality only when given consistent heuristics, due to the nature of the error-correction term $\dD$. Consequently, WBAE* requires heuristic consistency to ensure bounded suboptimality. Throughout the remainder of the paper, we therefore assume that the given heuristic $h$ is consistent.

\subsection{Theoretical analysis}
We first show that \bWae is bounded-suboptimal.
We then discuss the role of $\lambda$ and provide guidelines for selecting values for $\lambda$.
To prove that \bWae is bounded-suboptimal, we start by proving the following lemma.

\begin{lemma} \label{lemma:bw}
It holds that $\bWD(n) \leq W \cdot \bD(n)$ for every node $n$ and any $\lambda \leq W$.
\end{lemma}

\begin{proof}
Let $n$ be a node discovered during the search in direction $D$.
We want to show that $\bWD(n) \leq W \cdot \bD(n)$.
By using the definitions of $\bWD$ and $\bD$, we get:
\begin{equation*}
    \bWD(n) = \gD(n) + W \cdot \hD(n) + \lambda \cdot (\gD(n) - \hO(n))
\end{equation*}
and similarly,
\begin{equation*}
    W \cdot \bD(n) = W \cdot (\gD(n) + \hD(n) + (\gD(n) - \hO(n)))
\end{equation*}

Substituting these into the inequality in Lemma~\ref{lemma:bw}, $\bWD(n) \leq W \cdot \bD(n)$, we now need to prove that:
\begin{align*}
    \gD(n)& + W \cdot \hD(n) + \lambda \cdot (\gD(n) - \hO(n)) \leq \\
    &W \cdot \gD(n) 
     + W \cdot \hD(n) 
     + W \cdot (\gD(n) - \hO(n))
\end{align*}

By eliminating $W \cdot \hD(n)$ from both sides and simplifying, the inequality that we need to prove becomes:
\begin{equation}\label{eq:final}
    (W - \lambda) \hO(n) \leq (2W - \lambda - 1) \gD(n)
\end{equation} 

Since \(W \geq \lambda\) and \(W \geq 1\) by definition, it follows that \((2W - \lambda - 1) > 0\).
Additionally, because the heuristic is admissible, we have \(0 \leq \hO(n) \leq \gD(n)\).
Consequently, \((2W - \lambda - 1) \hO(n) \leq (2W - \lambda - 1) \gD(n)\).
To prove Inequality~\ref{eq:final}, it is therefore sufficient to demonstrate that:
\begin{equation} \label{eq:final2}
    (W-\lambda) \hO(n) \leq (2W - \lambda - 1) \hO(n)
\end{equation}
By simplifying the inequality, we get:
\begin{equation} \label{eq:final3}
    (1 - W) \hO(n) \leq 0
\end{equation}

For \(\hO(n) = 0\), inequality~\ref{eq:final3} holds trivially.
For \(\hO(n) > 0\), the inequality holds when \(W \geq 1\), which holds by definition.
Therefore, the inequality holds and \(\bWD(n) \leq W \cdot \bD(n)\) for all nodes \(n\).
\end{proof}

We proceed by proving the bounded-suboptimality of \bWae by making use of Lemma~\ref{lemma:bw}.
\begin{theorem}
\bWae is bounded-suboptimal.
\end{theorem}
\begin{proof}
Previous analysis of BAE*~\citep{sadhukhan2013bidirectional} shows that when the heuristic function is consistent, it holds that $\frac{\bminf + \bminb}{2} \leq \Cstar$ throughout the search.
Consequently, it is sufficient to show that throughout the search, $\frac{\bWminf + \bWminb}{2} \leq W \cdot \frac{\bminf + \bminb}{2}$, as this would imply that $\frac{\bWminf + \bWminb}{2} \leq W \cdot \Cstar$.

Let $n_f$ and $n_b$ be the nodes with the minimal $\bF$ (\bminf) and $\bB$ (\bminb) values, respectively, at some arbitrary iteration during the search.
From Lemma~\ref{lemma:bw}, we know that $\bWF(n_f) \leq W \cdot \bF(n_f)$ and $\bWB(n_b) \leq W \cdot \bB(n_b)$.
Thus,
\begin{align*}
\tfrac {\bWminf + \bWminb}{2} &\leq
    \tfrac{\bWF(n_f) + \bWB(n_b)}{2} \\
    &\leq W \cdot \tfrac{\bF(n_f) + \bB(n_b)}{2} \\
    &= W \cdot \tfrac{\bminf + \bminb}{2}
    \leq W\cdot \Cstar.
\end{align*}

Since \bWae terminates only when a solution is found with cost  $\leq \frac{\bWminf + \bWminb}{2}$, it is guaranteed to return a solution with a cost  $\leq W \cdot \Cstar$ (\bss)
This ensures that \bWae is a bounded-suboptimal algorithm.
\end{proof}

Finally, we show that similar to \wastar, \bWae can avoid node re-expansions when given consistent heuristics.
\begin{theorem}
Given consistent heuristics, \bWae finds a bounded-suboptimal solution without re-expanding nodes.
\end{theorem}
\begin{proof}
Due to the priority function of \bWae, every node $n$ in direction $D$ that undergoes expansion is ensured to have been reached through a bounded suboptimal path, specifically satisfying, $\gD(n) \leq W \cdot c(\start,n)$.
For a more detailed proof, refer to Appendix B.
\end{proof}

\subsection[The Role of lambda]{The Role of $\lambda$}
In this section, we discuss the role of $\lambda$ and provide guidelines for selecting an appropriate value for it.

In heuristic search, algorithms must perform two parallel tasks.
The first task is finding and refining solutions, which corresponds to lowering the upper bound ($U$).
The second task is proving the (sub)optimality of solutions, corresponding to increasing the lower bound ($LB$).
Search algorithms can only terminate when $U \leq LB$, making both tasks essential.
However, the approaches for tackling each task often contradict one another.
For example, the DVCBS algorithm \citep{shperberg2019enriching} excels at increasing the lower bound but often finds solutions more slowly than other algorithms, whereas Greedy Best-First Search (GBFS)~\cite{PUZ1566}---which prioritizes nodes based solely on their \(h\)-values---is good at finding solutions, but is not concerned with proving optimality.

We next show that while incorporating the heuristic error $d$ supports the task of proving solution suboptimality (the second task), it can hinder the process of finding a solution (the first task).
Given two nodes, $n_1$ and $n_2$, in direction $D$ with $\gD(n_1) + W \cdot \hD(n_1) = \gD(n_2) + W \cdot \hD(n_2)$, \bWae with $\lambda > 0$ delays the expansion of the node with the higher heuristic error $d$.
By definition, $d(n) = \gD(n) - \hO(n)$.
Due to admissibility, $\gD(n) \geq \hO(n) \geq 0$.
Thus, for every node $n$, $d(n) \leq \gD(n)$.
Consequently, nodes with higher $g$-values tend to have higher $d$ values, and their expansion is often delayed by \bWae.
Expanding nodes with high $g$-values is essential for finding solutions.
Therefore, \bWae is expected to be less efficient than \bwa in finding solutions.
Nonetheless, $d$ offers a more accurate estimate of the lower bound on solution cost, making a larger $\lambda$ more effective for establishing the (sub)optimality of solutions.

Due to the trade-off between increasing \(LB\) and decreasing \(U\), different approaches are required based on the specific challenge of each task.
At one extreme, if the initial heuristic yields a sufficiently high lower bound to guarantee (sub)optimality upon finding a solution, the focus should be on finding a solution---favoring a small (or zero) $\lambda$, especially when valid solutions are rare.
Conversely, when many (sub)optimal solutions exist but proving their quality requires extensive search, the emphasis should be on raising $LB$, implying a large $\lambda$.
In general, we expect the task of increasing \(LB\) to be dominant when: (i) the heuristic is weak, significantly underestimating the shortest paths from many states, and (ii) \(W\) is small, as the effort to prove suboptimality increases.
Thus, we propose the following hypothesis:

\smallskip
\noindent \textbf{Hypothesis on the role of \(\lambda\):} Larger values of \(\lambda\) are more effective when the heuristic is weak and the suboptimality bound \(W\) is small. Smaller values of \(\lambda\) are preferable when the heuristic is more accurate or when \(W\) is large.

Figure~\ref{fig:example} presents a small example illustrating how different values of $\lambda$ affect node expansions during search.
The top displays a problem instance with heuristic values and edge costs.
We consider \bWae with $W=1.1$ and $W=5$, and two $\lambda$-values, $\lambda = W$ and $\lambda = W^{-2}$; priorities for each configuration are given inside the table (excluding nodes $s$ and $g$, which are expanded first regardless).

All algorithms first expand $s$ and $g$. Now, consider $w = 1.1$. With both values of $\lambda$, $u$ is expanded, yielding a solution of cost 5. For $\lambda = W$, the search can terminate: the lowest priority in the backward direction is 6.07, and in the forward direction, it is at most the priority of $x$, which is 4.21. Hence, the lower bound $(6.07+4.21)/2=5.14$ exceeds 5. In contrast, with $\lambda = W^{-2}$, after $u$ is expanded, the minimal forward priority is at most that of $x$, giving a lower bound of $(5.52 + 4.07)/2 = 4.795$, which is less than 5, requiring the expansion of $x$.
For $w = 5$ and $\lambda = W^{-2}$, $u$ is expanded and the search terminates. With $\lambda = W$, $x$ is expanded even before $u$, leading to more node expansions. 

This example supports our hypothesis above. For small values  of $w$  ($w=1.1$) larger values of $\lambda$ ($\lambda=W$) perform better than smaller values ($\lambda=W^{-2}$). For large values  of $W$ ($W=5$) small values of $\lambda$ ($\lambda=W^{-2}$) perform better than larger values ($\lambda=W$).

\begin{figure}[t]
\centering
\subfloat{\includegraphics[width=0.8\linewidth]{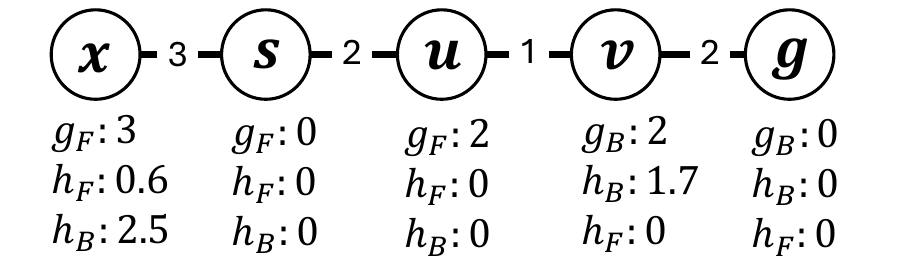}} \\
\smallskip
\small
\begin{tabular}{ll|rr|rr}
\toprule
\multirow{2}{*}{State} & \multirow{2}{*}{Dir.} & \multicolumn{2}{c|}{$W=1.1$} & \multicolumn{2}{c}{$W=5$} \\ \cline{3-6} & & $\lambda = W$ & $\lambda =W^{-2}$ & $\lambda =W$ & $\lambda =W^{-2}$ \\
\midrule
$u$ & F & 4.20 & 3.65 & 12.0 & 2.08 \\
$x$ & F & 4.21 & 4.07 & 8.5 & 6.02 \\
$v$ & B & 6.07 & 5.52 & 20.5 & 10.58 \\
\bottomrule
\end{tabular}
\caption{An example demonstrating the role of $\lambda$}
\label{fig:example}
\end{figure}

\section{Obtaining Stronger Lower Bounds}

In this section, we examine two methods to tighten lower bounds throughout the search, potentially reducing the effort needed by \bss algorithms to return a suboptimal solution.

\subsection{Alternative Termination Criteria}

As reviewed in the background section, algorithms in the \wastar family typically terminate based on the minimum values of their priority, which multiply the heuristic estimation $h$ (but not the cost $g$) by $W$.
For example, \bwa terminates when $U \leq \max(\prWminf , \prWminb)$, where $U$ is the cost of the incumbent solution, and \prWminf, \prWminb are the minimal $\prW$-values in the open lists.

Nonetheless, \fminf and \fminb are still lower bounds on \Cstar, the cost of the shortest path between \start and \goal.
Consequently, $W \cdot \max(\fminf, \fminb) \leq W \cdot \Cstar$, and therefore $U \leq W \cdot \max(\fminf, \fminb)$ is an alternative termination condition that still maintains the bounded-suboptimality guarantees.
Moreover, let $n$ be the node with the minimal $f$-value in direction $D$.
Thus, 
\begin{align*}
    W \cdot \fmind &= W \cdot \gD(n) + W \cdot \hD(n) \\
    &\geq \gD(n) + W \cdot \hD(n)  \\
    &\geq \prWD(n) \geq \prWmind
\end{align*}

This means that the alternative lower bound (denoted \alb) is tighter than the original lower bound, which can result in fewer expansions during the search.

The above \alb was defined based on \fmin.
Nonetheless, a similar \alb can be defined for all algorithms in the \wastar family by taking the bounds of their corresponding optimal variant and multiplying them by $W$.
For \bWae, \alb will therefore halt the search when $U \leq W \cdot LB_{B}$. 

Nevertheless, using \alb incurs extra overhead as it maintains both priorities (e.g., \prWmind and \fmind for \bwa, or \bWmind and \bmind for \bWae).
This might require storing two open lists for each frontier.
We analyze the improvement resulting from \alb, as well as the computational overhead, in our empirical evaluation below.

\subsection{Utilizing Information on Edge Costs}

\begin{table}[tb]
\begin{adjustbox}
{width=0.9\columnwidth,center}
\begin{tabular}{lrrrrr}
\toprule
\multicolumn{1}{c}{Algorithm} & \multicolumn{1}{c}{ToH} & \multicolumn{1}{c}{DAO} & \multicolumn{1}{c}{Mazes} & \multicolumn{1}{c}{Pancake} & \multicolumn{1}{c}{STP} \\
\midrule
\wastar & 697K & 1,585K & 2,024K & 56K & 524K \\
\bwa & 601K & 1,466K & 1,694K & 76K & 502K \\
\wmm & 589K & 845K & 606K & 75K & 400K \\
\WBSstar  & 644K & 1,564K & 1,985K & 67K & 497K \\
\cmidrule(lr){1-6}
\bWae$\frac{1}{W^2}$ & 607K & 1,447K & 1,603K & 66K & 465K\\
\bWae$\frac{1}{W}$ & 540K & 1,447K & 1,646K & 66K & 450K\\
\bWae$1$ & 534K & 1,468K & 1,717K & 65K & 428K\\
\bWae$W$ & 523K & 1,464K & 1,667K & 61K & 336K\\
\bWae $\lambda$ & 478K & 1,454K & 1,709K & 60K & 430K\\
\bottomrule
\end{tabular}
\end{adjustbox}
\caption{Expansions/sec for all domains and algorithms, averaged across all weights $W$.}
\label{tb:nps}
\end{table}

\begin{table*}[tb]
\begin{adjustbox}{width=\textwidth,center}
\begin{tabular}{@{}l|ccccccccc|ccccccccc|ccccccccc@{}}
\toprule
\multicolumn{1}{c|}{} & \multicolumn{9}{c|}{ToH-12 (10+2)} & \multicolumn{9}{c|}{ToH-12 (8+4)} & \multicolumn{9}{c}{ToH-12 (6+6)} \\ \midrule
\multicolumn{1}{c|}{Algorithm} & 1 & 1.1 & 1.2 & 1.5 & 1.7 & 2 & 3 & 5 & 10 & 1 & 1.1 & 1.2 & 1.5 & 1.7 & 2 & 3 & 5 & 10 & 1 & 1.1 & 1.2 & 1.5 & 1.7 & 2 & 3 & 5 & 10 \\
\cmidrule(lr){2-10}
\cmidrule(lr){11-19}
\cmidrule(lr){20-28}
\wastar & 279K & 137K & 54K & 4K & 2K & 1K & 417 & 433 & 441 & 2M & 2M & 1M & 486K & 318K & 278K & 26K & 13K & 20K & 3M & 3M & 3M & 2M & 2M & 1M & 1M & 863K & 331K \\
\bwa & 236K & 106K & 40K & \textbf{3K} & \textbf{1K} & \textbf{627} & \textbf{374} & \textbf{343} & 347 & 1M & 806K & 558K & 186K & 109K & 39K & \textbf{11K} & \textbf{5K} & \textbf{4K} & 2M & 2M & 1M & 682K & 439K & 266K & 117K & 72K & \textbf{43K} \\
\wmm & 337K & 164K & 60K & 3K & 2K & 2K & 5K & 4K & 556 & 1M & 863K & 707K & 269K & 157K & 56K & 15K & 10K & 8K & 964K & 1M & 1M & 842K & 600K & 395K & 251K & 121K & 64K \\
\WBSstar & 179K & 82K & 33K & \textbf{3K} & \textbf{1K} & 649 & 377 & 349 & 353 & 924K & 656K & 466K & 167K & 94K & 35K & \textbf{11K} & 6K & \textbf{4K} & 2M & 1M & 1M & 584K & 390K & 248K & \textbf{115K} & 76K & 49K \\ \cmidrule(lr){2-28}
\bWae$\frac{1}{W^2}$ & \textbf{48K} & 30K & \textbf{17K} & 5K & 3K & 1K & 394 & 353 & 343 & \textbf{189K} & 182K & 162K & 105K & 74K & 34K & 12K & \textbf{5K} & \textbf{4K} & \textbf{382K} & 403K & 402K & 358K & 320K & \textbf{221K} & 119K & \textbf{71K} & 45K \\
\bWae$\frac{1}{W}$ & \textbf{48K} & \textbf{28K} & 18K & 7K & 4K & 1K & 405 & 364 & 343 & \textbf{189K} & 162K & 137K & \textbf{97K} & 76K & 42K & 14K & 6K & \textbf{4K} & \textbf{382K} & 362K & 342K & 310K & 284K & 227K & 130K & 77K & 47K \\
\bWae$1$ & \textbf{48K} & \textbf{28K} & 20K & 9K & 6K & 3K & 536 & 380 & 357 & \textbf{189K} & 148K & \textbf{130K} & 105K & 92K & 69K & 22K & 9K & 5K & \textbf{382K} & 333K & 317K & 301K & 292K & 263K & 176K & 111K & 68K \\
\bWae$W$ & \textbf{48K} & 29K & 23K & 13K & 11K & 7K & 3K & 2K & 1K & \textbf{189K} & \textbf{143K} & 131K & 109K & 102K & 88K & 71K & 56K & 41K & \textbf{382K} & \textbf{318K} & \textbf{306K} & \textbf{282K} & \textbf{273K} & 252K & 227K & 202K & 186K \\
\bWae$\lambda^*$ & \textbf{48K} & \textbf{28K} & \textbf{17K} & \textbf{3K} & 2K & 789 & 394 & 356 & \textbf{342} & 191K & 144K & 132K & 98K & \textbf{72K} & \textbf{32K} & \textbf{11K} & \textbf{5K} & \textbf{4K} & 386K & 319K & 309K & 309K & 284K & 223K & 121K & \textbf{71K} & 45K \\ \bottomrule
\end{tabular}
\end{adjustbox}
\caption{Average number of node expansions on the 12-disks ToH domain, with $(10+2)$, $(8+4)$, and $(6+6)$ PDBs}
\label{tab:hanoi}
\end{table*}

\begin{table*}[tb]
\begin{adjustbox}{width=\textwidth,center}
\begin{tabular}{@{}l|ccccccccc|ccccccccc|ccccccccc@{}}
\toprule
\multicolumn{1}{c|}{} & \multicolumn{9}{c|}{STP} & \multicolumn{9}{c|}{Mazes} & \multicolumn{9}{c}{DAO} \\ \midrule
\multicolumn{1}{c|}{Algorithm} & 1 & 1.1 & 1.2 & 1.5 & 1.7 & 2 & 3 & 5 & 10 & 1 & 1.1 & 1.2 & 1.5 & 1.7 & 2 & 3 & 5 & 10 & 1 & 1.1 & 1.2 & 1.5 & 1.7 & 2 & 3 & 5 & 10 \\ 
\cmidrule(lr){2-10} \cmidrule(lr){11-19} \cmidrule(lr){20-28}
\wastar & 16M & 10M & 3M & 320K & 103K & 41K & 12K & \textbf{5K} & 4K & 99K & 98K & 96K & 92K & 90K & 86K & 76K & 65K & \textbf{56K} & \textbf{5K} & \textbf{5K} & \textbf{4K} & \textbf{4K} & \textbf{3K} & \textbf{3K} & \textbf{2K} & \textbf{2K} & \textbf{2K} \\
\bwa & 19M & 4M & 658K & 78K & 50K & \textbf{29K} & \textbf{11K} & \textbf{5K} & 4K & 100K & 99K & 98K & 95K & 92K & 89K & 77K & 66K & 59K & 7K & 6K & 5K & \textbf{4K} & 4K & \textbf{3K} & 3K & 3K & \textbf{2K} \\
\wmm & 15M & 4M & 861K & 162K & 160K & 270K & 2M & 3M & 47K & 85K & 85K & 85K & 84K & 84K & 84K & 80K & 70K & 60K & 8K & 7K & 6K & 5K & 5K & 4K & 3K & 3K & \textbf{2K} \\
\WBSstar & 12M & 3M & 486K & \textbf{76K} & \textbf{49K} & \textbf{29K} & \textbf{11K} & \textbf{5K} & 4K & 88K & 87K & 86K & 83K & 82K & 78K & \textbf{69K} & \textbf{60K} & \textbf{56K} & 6K & \textbf{5K} & 5K & \textbf{4K} & 4K & \textbf{3K} & 3K & 3K & 3K \\ \cmidrule(lr){2-28}
\bWae$\frac{1}{W^2}$ & \textbf{3M} & \textbf{1M} & 449K & 143K & 80K & 34K & 12K & \textbf{5K} & \textbf{3K} & \textbf{81K} & 85K & 88K & 90K & 90K & 87K & 77K & 66K & 59K & 7K & 6K & 6K & 5K & 4K & 4K & 3K & 3K & \textbf{2K} \\
\bWae$\frac{1}{W}$ & \textbf{3M} & \textbf{1M} & 586K & 202K & 139K & 57K & 16K & 6K & \textbf{3K} & \textbf{81K} & 83K & 84K & 85K & 84K & 83K & 76K & 66K & 59K & 7K & 6K & 6K & 5K & 4K & 4K & 3K & 3K & \textbf{2K} \\
\bWae$1$ & \textbf{3M} & \textbf{1M} & 763K & 337K & 250K & 171K & 40K & 11K & 5K & \textbf{81K} & \textbf{80K} & \textbf{80K} & \textbf{79K} & \textbf{78K} & \textbf{77K} & 74K & 69K & 62K & 7K & 6K & 6K & 5K & 5K & 4K & 3K & 3K & \textbf{2K} \\
\bWae$W$ & \textbf{3M} & \textbf{1M} & 950K & 688K & 617K & 553K & 1M & 831K & 575K & \textbf{81K} & \textbf{80K} & \textbf{80K} & 80K & 79K & 79K & 78K & 77K & 77K & 7K & 6K & 6K & 5K & 5K & 5K & 4K & 4K & 4K \\ 
\bWae$\lambda^*$ & \textbf{3M} & \textbf{1M} & \textbf{368K} & 89K & 60K & 33K & 12K & \textbf{5K} & \textbf{3K} & 81K & \textbf{80K} & \textbf{80K} & \textbf{79K} & \textbf{78K} & \textbf{77K} & 73K & 66K & 59K  & 7K & 6K & 5K & \textbf{4K} & 4K & \textbf{3K} & 3K & 3K & \textbf{2K} \\

\bottomrule
\end{tabular}
\end{adjustbox}
\caption{Average number of node expansions on the STP, Mazes, and DAO domains}
\label{tab:SMD}
\end{table*}

\begin{table*}[tb]
\begin{adjustbox}{width=\textwidth,center}
\begin{tabular}{@{}l|ccccccccc|ccccccccc|ccccccccc@{}}
\toprule
\multicolumn{1}{c}{} & \multicolumn{9}{c}{GAP} & \multicolumn{9}{c}{GAP-1} & \multicolumn{9}{c}{GAP-2} \\ \midrule
\multicolumn{1}{c|}{Algorithm} & 1 & 1.1 & 1.2 & 1.5 & 1.7 & 2 & 3 & 5 & 10 & 1 & 1.1 & 1.2 & 1.5 & 1.7 & 2 & 3 & 5 & 10 & 1 & 1.1 & 1.2 & 1.5 & 1.7 & 2 & 3 & 5 & 10 \\ \cmidrule(lr){2-10} \cmidrule(lr){11-19} \cmidrule(lr){20-28}
\wastar & \textbf{194} & 147 & \textbf{51} & \textbf{28} & \textbf{26} & \textbf{23} & \textbf{23} & \textbf{23} & \textbf{23} & 42K & 25K & 3K & 271 & 146 & 62 & 46 & \textbf{42} & \textbf{42} & 5M & 5M & 1M & 77K & 21K & 4K & 1K & 1K & 1K \\
\bwa & 225 & \textbf{100} & 60 & 41 & 40 & 37 & 37 & 37 & 37 & 61K & 5K & 893 & 59 & 49 & \textbf{43} & \textbf{42} & \textbf{42} & \textbf{42} & 8M & 2M & 194K & 5K & 1K & 349 & 180 & \textbf{164} & \textbf{167} \\
\wmm & 417 & 245 & 312 & 2K & 5K & 19K & 5K & 55 & 37 & 43K & 4K & 1K & 2K & 4K & 10K & 103K & 665 & 54 & 3M & 510K & 90K & 12K & 5K & 16K & 134K & 46K & 437 \\
\WBSstar & 203 & 106 & 62 & 41 & 40 & 37 & 37 & 37 & 37 & 33K & 4K & 884 & \textbf{51} & \textbf{48} & \textbf{43} & \textbf{42} & \textbf{42} & \textbf{42} & 4M & 1M & 126K & 4K & \textbf{527} & \textbf{254} & \textbf{175} & \textbf{164} & \textbf{167} \\ \cmidrule(lr){2-28}
\bWae$\frac{1}{W^2}$ & 205 & 111 & 78 & 43 & 41 & 37 & 36 & 36 & 36 & \textbf{2K} & \textbf{596} & 386 & 75 & 59 & 48 & 43 & \textbf{42} & \textbf{42} & 145K & 48K & 9K & 3K & 1K & 540 & 180 & 176 & 180 \\
\bWae$\frac{1}{W}$ & 205 & 114 & 80 & 46 & 42 & 37 & 36 & 36 & 36 & \textbf{2K} & 637 & 511 & 120 & 71 & 50 & 43 & \textbf{42} & \textbf{42} & 145K & 38K & \textbf{8K} & 7K & 2K & 818 & 176 & 176 & 180 \\
\bWae$1$ & 205 & 116 & 80 & 48 & 43 & 38 & 36 & 36 & 36 & \textbf{2K} & 687 & 580 & 164 & 117 & 60 & 45 & \textbf{42} & \textbf{42} & 145K & 37K & 11K & 12K & 9K & 2K & 336 & 197 & 192 \\
\bWae$W$ & 205 & 115 & 83 & 57 & 55 & 44 & 47 & 55 & 47 & \textbf{2K} & 742 & 873 & 2K & 1K & 845 & 385 & 216 & 178 & 145K & 37K & 17K & 71K & 132K & 393K & 85K & 44K & 26K \\ 
\bWae $\lambda^*$ & 205 & 114 & 66 & 42 & 40 & 37 & 36 & 36 & 36 & \textbf{2K} & 583 & \textbf{283} & 75 & 60 & 50 & 43 & \textbf{42} & \textbf{42} & \textbf{134K} & \textbf{35K} & \textbf{8K} & \textbf{2K} & 1K & 540 & 196 & 172 & 201 \\ \bottomrule

\end{tabular}
\end{adjustbox}
\caption{Average number of node expansions on 18 Pancakes domain, with GAP to GAP-2 heuristics}
\label{tab:pancake18}
\end{table*}

\citet{alcazar20unifying} demonstrated how edge cost information can tighten lower bounds on \Cstar throughout the search.
They assumed the algorithm knows the  {\em greatest common denominator} (GCD) of all edge costs, $\iota$, and used it to redefine \bae's termination condition:
\begin{equation*}
    LB_B = \left\lceil \tfrac{{(\bminf + \bminb)}/{2}}{\iota} \right\rceil \cdot \iota
\end{equation*}
and showed that this new definition is a lower bound on \Cstar, based on the idea that solutions can only increase in GCD increments.
Thus, if a lower bound suggests a solution not divisible by $\iota$, it can be rounded up to the nearest $\iota$ increment.
For example, if $\iota=3$ and the lower-bound $LB=14$, $\lceil 14/3 \rceil \cdot 3 = 15$ can be used as a tighter lower bound.
 A similar approach could apply to \astar's lower bound, but typically, $\fmin = g(n) + h(n)$ is already divisible by $\iota$ since $g(n)$ is always divisible by $\iota$, and $h(n)$ is often as well.

This alteration to the lower bound formula can also be applied to the \bss variants of \bae as follows:
\begin{equation*}
    LB_{\mathtt{WB}} = \left\lceil \tfrac{{(\bWminf + \bWminb)}/{2}}{\iota \cdot W} \right\rceil \cdot \iota \cdot W
\end{equation*}

Let $n_f$ and $n_b$ be the nodes with the minimal $b$-value in \openf and \openb, respectively, at some arbitrary iteration during the search.
This formulation of $LB_{\mathtt{WB}}$ still yields bounded solutions, since:
\begin{adjustbox}{width=\linewidth,center}
\begin{minipage}{\linewidth}
\begin{align*}
    \left\lceil \tfrac{(\bWminf + \bWminb)/{2}}{\iota \cdot W} \right\rceil \cdot \iota \cdot W 
    &\leq \left\lceil \tfrac{(\bWF(n_f) + \bWF(n_B))/{2}}{\iota \cdot W} \right\rceil \cdot \iota \cdot W \\
    &\leq \left\lceil \tfrac{(W \cdot {(\bF(n_f) + \bF(n_B))}/{2}}{\iota \cdot W} \right\rceil \cdot \iota \cdot W \\
    &= \left\lceil \tfrac{W \cdot (\bminf + \bminb)/{2}}{\iota \cdot W} \right\rceil \cdot \iota \cdot W \\
    &= \left\lceil \tfrac{(\bminf + \bminb)/{2}}{\iota} \right\rceil \cdot \iota \cdot W \leq \Cstar \cdot W
\end{align*}
\end{minipage}
\end{adjustbox}

Similar adjustments can be applied to all \wastar-family algorithms.
Notably, for $W>1$, this revised formulation can benefit \bwa (and also \wastar), as $\fmin=g(n)+W \cdot h(n)$ is not guaranteed to be divisible by $W$.
Moreover, unlike ALB, which requires maintaining multiple open lists that incur additional computational overhead, the GCD improvement adds no computational cost.

\section{Empirical Evaluation}

We now present an empirical evaluation comparing \bihs \bss algorithms, including our newly introduced \bWae.
This evaluation also tests our hypothesis on the role of $\lambda$ and assesses the impact of the alternative termination conditions.
All code for reproducing the results—including algorithms, domain implementations, and experiment scripts—will be available upon acceptance (*anonymized for review*).


\begin{figure*}[t]
    \centering
    \subfloat[GCD Reduction in Expansions]{%
        \includegraphics[width=0.29\linewidth]{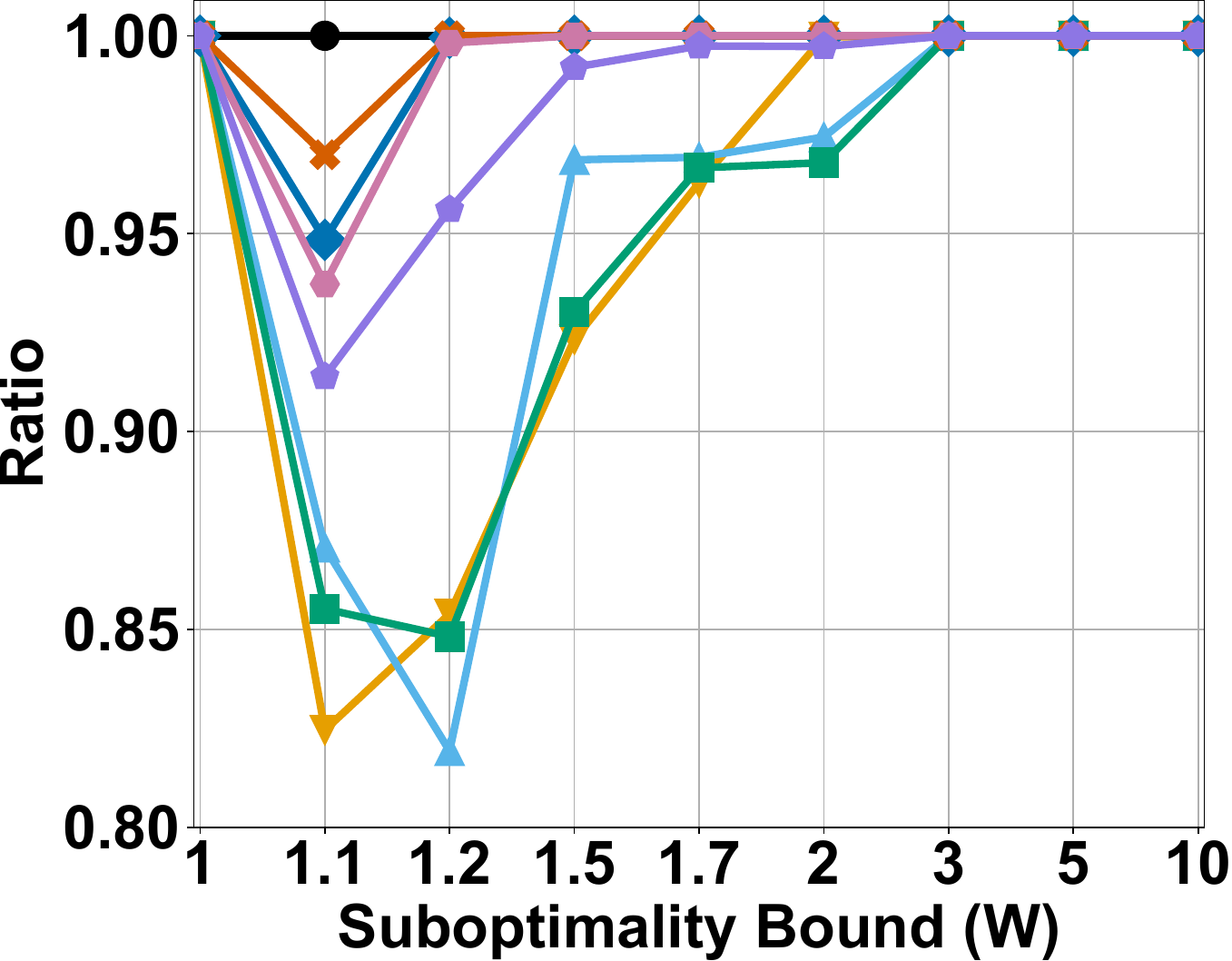}
        \label{fig:subfig3}
    }
    \hfill
    \subfloat[\alb Reduction in Expansions]{%
        \includegraphics[width=0.29\linewidth]{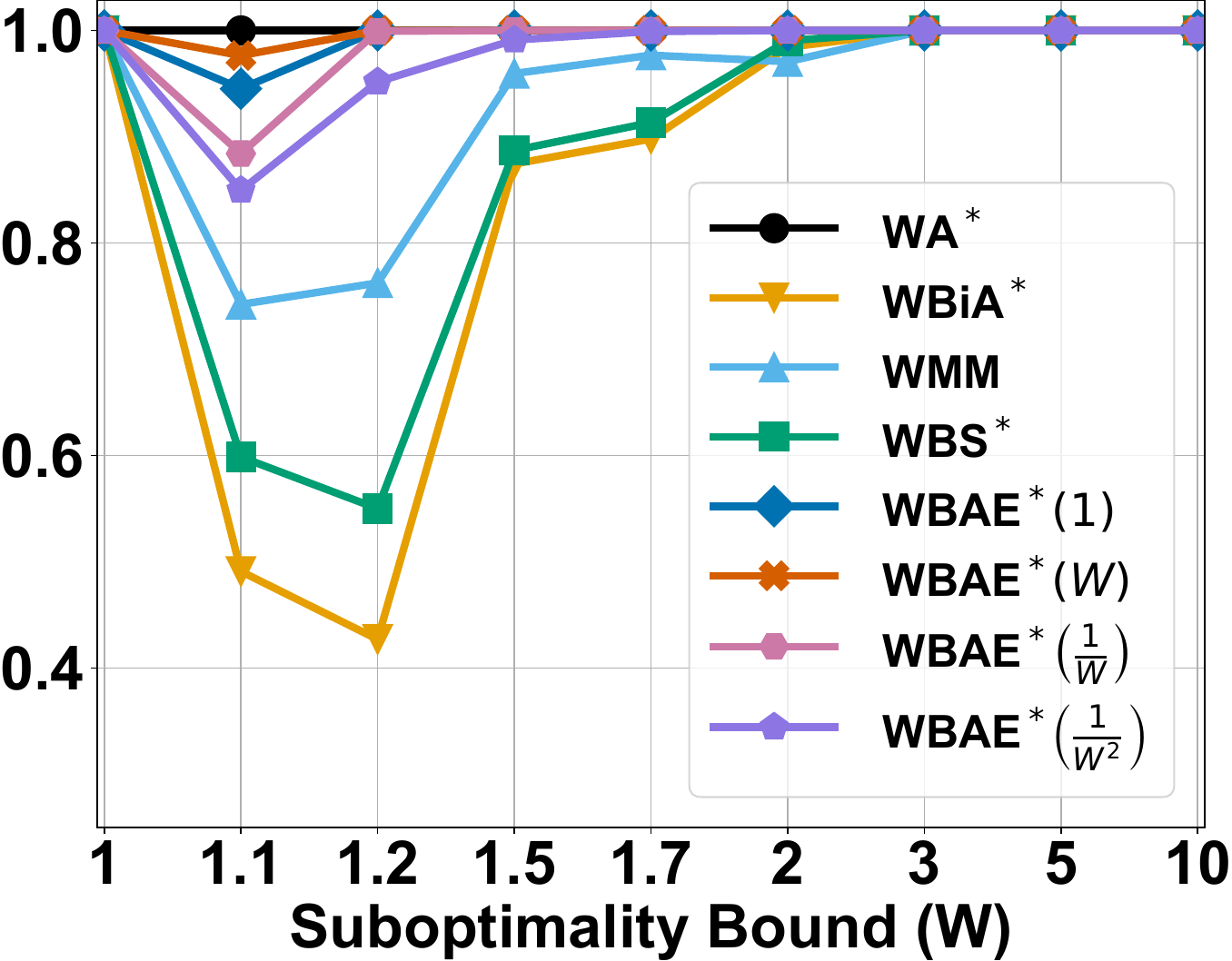}
        \label{fig:subfig1}
    }
    \hfill
    \subfloat[\alb Reduction in Runtime]{%
        \includegraphics[width=0.3\linewidth]{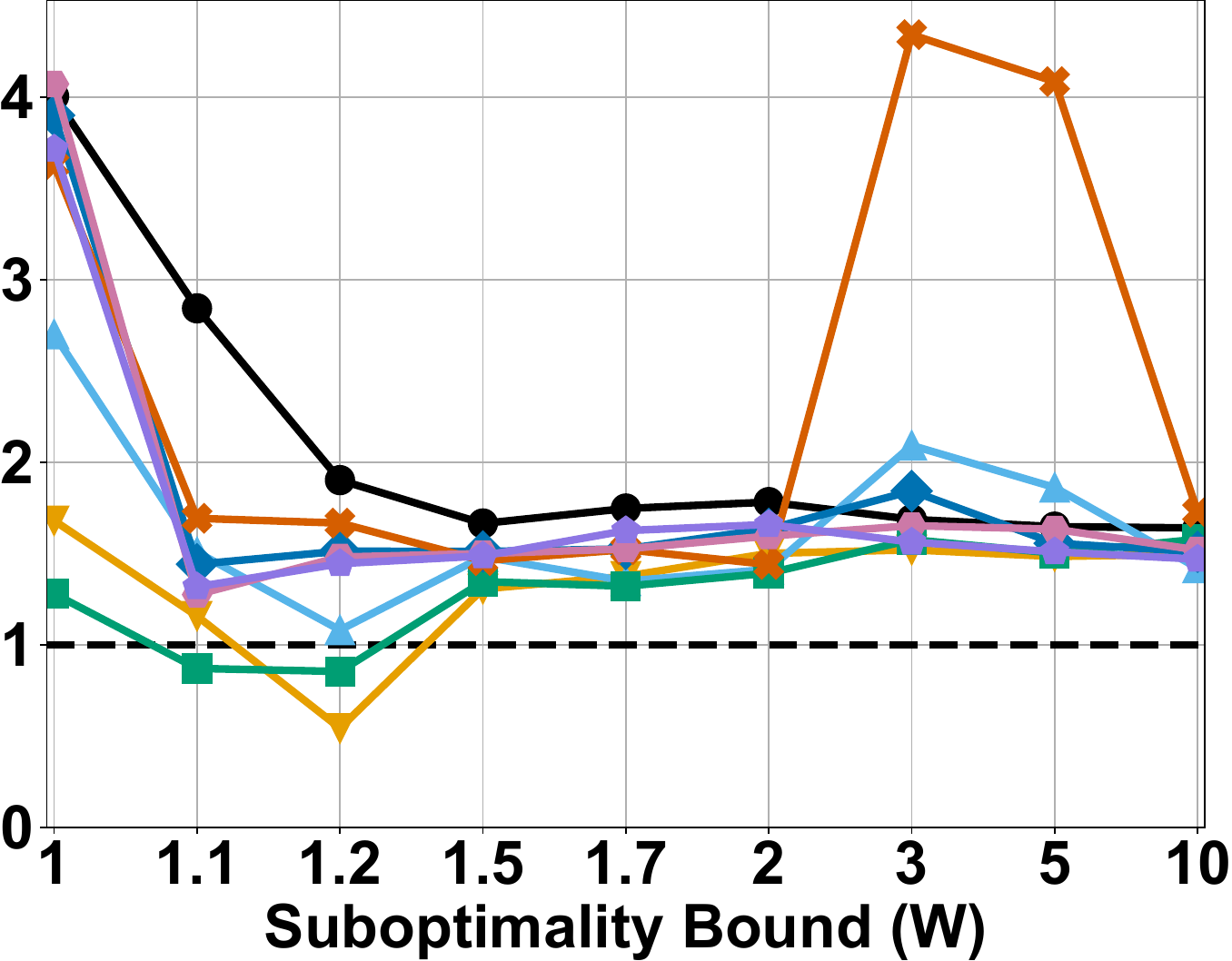}
        \label{fig:subfig2}
    }
    
    \caption{Comparison of the methods for strengthening the lower bound on STP}
    \label{fig:three_side_by_side}
\end{figure*}

\smallskip
\noindent{\bf Baselines.}
We compare all existing \bihs \bss algorithms: \bwa, \WBSstar, \wmm, and \bWae, as well as \wastar as a representative \UniHS \bss algorithm. For \bWae, we evaluate several values of \(\lambda \in \{0, \frac{1}{W^2}, \frac{1}{W}, 1, W\}\) to examine its impact on performance. Since \bWae with \(\lambda = 0\) uses the same priority rule as \bwa, we omit it from the results. All algorithms break ties in favor of higher \(g\)-values.



In addition, we report results for a \emph{tuned} \(\lambda\) value, denoted \(\lambda^*\), selected separately for each domain and each \(W\). To determine \(\lambda^*\), we used Optuna~\cite{AkibaSYOK19}, a hyperparameter optimization framework that automates the process of finding the best configuration parameters for a model. Optuna performed 50 trials; in each, it selected a \(\lambda\) value and ran \bWae on an additional set of 100 domain-specific instances disjoint from the test set. The value with the best overall tuning performance was used for testing. The resulting \(\lambda^*\) values are listed in Appendix~F.

\smallskip
\noindent{\bf Domains.}
We experimented on five domains: {\bf (1)} 100 instances of the 12-disk {\em 4-peg Towers of Hanoi} problem (ToH) with (10+2), (8+4), and (6+6) additive Pattern Databases (PDBs)~\citep{Felner2004}.
{\bf (2)} The standard 100 instances of the \emph{15 puzzle} problem (STP)~\citep{Korf1985} using the Manhattan distance heuristic.
{\bf (3)} 100 instances of the 18-pancake puzzle with the GAP heuristic~\citep{Helmert2010}.
To explore different heuristic strengths, we also used GAP-$n$ heuristics (for $n=1,2$) where the $n$ smallest pancakes are excluded from the heuristic computation.
We also experimented on the following \emph{Grid}-based pathfinding maps from the \emph{MovingAI} repository~\citep{sturtevant2012benchmarks} using octile distance as a heuristic and 1.5 cost for diagonal edges: {\bf (4)} 156 maps from Dragon Age Origins (DAO), each with different start and goal points (totaling 3149 instances), and {\bf (5)} mazes with varying widths (1200 instances). Heuristic values for states were set to the maximum of the heuristic function's estimate and $1$, the minimum edge cost (commonly referred to as $\epsilon$). Throughout the evaluation, we explore various values of $W$ ranging from 1 to 10, specifically $W \in \{1, 1.1, 1.2, 1.5, 1.7, 2, 3, 5, 10\}$.

\smallskip
\noindent{\bf Runtime.} The different priority functions and the GCD enhancement had a small impact on the constant time per node.
Table~\ref{tb:nps} presents the average node expansions per second for each algorithm (including the GCD enhancement) across various domains.
Although there are some variations between domains, the overall overhead of node expansions is comparable, and the trends observed below in node expansions are largely preserved when considering runtime as well.
Comprehensive runtime results, broken down by heuristic and weight, can be found in Appendix C.

The \alb enhancement adds computational overhead and is not always time-efficient. Thus, the initial experiments include the GCD improvement but exclude \alb, as this combination was generally the most efficient. The impact of both GCD and \alb---especially on CPU time---is analyzed later.

\smallskip
\noindent{\bf ToH Results.}
Table~\ref{tab:hanoi} shows the average number of node expansions in the ToH domain using (10+2), (8+4), and (6+6) additive PDB heuristics.
In ToH, for optimal search (\(W=1\)), \bWae\ (effectively \BAE) significantly outperforms the other algorithms, and in particular, it outperforms \wastar\ (effectively \astar) by a factor of 6 to 10, depending on the heuristic.
This demonstrates the importance of using the heuristic error \(d\), when available, for proving the optimality of solutions.
However, in accordance with our hypothesis, as \(W\) increases, the focus of the search shifts from proving optimality to finding a solution.
Thus, the heuristic error $d$ becomes less important, and can even hinder the search.
For example, with the (10+2) PDB, \(\lambda=W\) is never the best policy for \(W>1\) among all the examined values.
For \(W=1.1\), we observe that \(\lambda = 1\) and \(\lambda =\frac{1}{W}\) have roughly the same performance (28K expansions), but for \(W=1.2\), a lower value of \(\lambda=\frac{1}{W^2}\) yields the best results.
For \(W \geq 1.5\), it is better to ignore \(d\) altogether and use $\lambda=0$ (= \bwa).

For the weaker heuristics, (8+4) and (6+6), the trends are similar, although the transition point between the best values of \(\lambda\) is different.
For example, with the (8+4) heuristic, \(\lambda=W\) is best performing up to \(W=1.1\), while with the (6+6) heuristic, \(\lambda=W\) is better up to \(W=2\).
This observation is also in agreement with our hypothesis, as we see that when the heuristic gets weaker, higher values of \(\lambda\) tend to work better.
Notably, in this domain, \wastar\ was never the best choice, although the advantage of the \bihs\ algorithm diminishes as \(W\) increases.
Moreover, neither \wmm\ nor \WBSstar\ emerged as the best algorithms. \bwa\ (similar to \bWae with $\lambda = 0$) outperformed all other variants in some configurations with larger values of $W$, particularly when using the stronger heuristic (10+2). However, it struggled with smaller values of $W$, especially when combined with weaker heuristics (8+4 and 6+6). This pattern aligns with our hypothesis: when heuristics are strong and $W$ is large, less effort is needed to validate (sub)optimality, diminishing the impact of $\lambda$. Notably, the tuned \(\lambda\) values, \(\lambda^*\), achieved the best average performance across weights, often matching or outperforming the best fixed \(\lambda\) configurations.

\smallskip
\noindent{\bf STP, Mazes, and DAO.}
Table~\ref{tab:SMD} shows the results for STP, Mazes, and DAO. For $W = 1$, \bWae\ outperforms \wastar\ (and all other algorithms) by a factor of 5. However, as $W$ increases, the importance of $d$ diminishes, and lower $\lambda$ values become more effective. Notably, a small value of $\lambda = \frac{1}{W^2}$ performs well even for small $W$, due to the relatively strong MD heuristic.
To test this effect with a weaker heuristic, we repeated the experiment using MD-4 (ignoring four central tiles, like gap-4). Full results are in Appendix~E. Here, larger \(\lambda\) values performed better; for instance, with MD-4 and \(W=1.2\), \(\lambda=1\) led to fewer expansions than \(\lambda = \frac{1}{W^2}\), supporting our hypothesis.

For $W \geq 1.5$, \bwa---which does not use $d$---becomes the most effective algorithm, until $W = 10$, where $\lambda = \frac{1}{W}$ and $\lambda = \frac{1}{W^2}$ slightly outperform it. We also evaluated the heavy STP variant, a non-unit-cost domain where tile movement cost equals the tile's label. Results (see Appendix~D) follow similar trends; however, in this case, \bWae\ maintains a clear advantage over \wastar\ even at large $W$, reducing node expansions by a factor of 5 when $W = 10$.

For Mazes and DAO, the performance of {\em all} algorithms is relatively similar, particularly as \(W\) increases.
This is likely because the state space of these two domains is polynomial, which diminishes the differences between approaches, unlike in other domains where the state space is exponential. Nonetheless, even in Mazes, it is evident that smaller values of $\lambda$ become more effective as $W$ increases. For example, when $W = 1.1$, the larger values $\lambda = W$ and $\lambda = 1$ outperform the smaller values $\lambda = \frac{1}{W^2}$ and $\lambda = 0$, whereas the opposite holds for $W = 5$ and $W = 10$. This trend is consistent with our hypothesis as well.

In these domains as well, the tuned values \(\lambda^*\) consistently yielded strong performance, often matching or exceeding the best fixed \(\lambda\) configurations in terms of average results.

\smallskip
\noindent{\bf 18-pancake.}
Table~\ref{tab:pancake18} presents the results for the 18-pancake.
Here, \wastar\ shows superior performance for GAP, where the heuristic is highly accurate, as well as for large values of $W$ across all heuristics.
Nevertheless, the observed trend in the choice of $\lambda$ remains consistent: the best-performing value of $\lambda$ tends to decrease as heuristic accuracy increases and as $W$ becomes larger. Specifically, for GAP and GAP-1, smaller values such as $\lambda = \frac{1}{W^2}$ and $\lambda = 0$ perform as well as or better than all other $\lambda$ settings across all values of $W$. In contrast, for GAP-2, a weaker heuristic, larger values, such as $\lambda = W$ and $\lambda = 1$, are more effective when $W$ is small.

\smallskip
\noindent{\bf Solution Quality.} 
All algorithms produced solutions well within bounds. Even for $W=10$, average costs never exceeded $2.3 \cdot C^*$. Full results are available in Appendix C.

\subsection{Analysis of the Stronger Lower Bounds} 

Figures~\ref{fig:subfig3} and \ref{fig:subfig1} illustrate the ratio in node expansions achieved by GCD and \alb, respectively, relative to a variant that does not use them (the 1.0 line).
Additionally, the impact of \alb on the overall runtime is presented in Figure~\ref{fig:subfig2}.
These are representative results for the 15-puzzle (15-STP); other domains showed similar trends.
Notably, the GCD improvement imposes minimal computational overhead, so the reduction in node expansions directly corresponds to a decrease in runtime.
Both methods show the most significant reduction in nodes for small values of $W$.
For example, for $W \in \{1.1, 1.2\}$, GCD expanded approximately 85\% of the nodes for \bwa, \WBSstar, and \wmm (again, at no extra CPU cost).
\alb at $W=1.2$ achieves a reduction to 40\% nodes for \bwa and almost 50\% for \WBSstar.
However, the reduction is minor for other algorithms and becomes negligible when $W > 3$.
This trend highlights the diminishing complexity of proving solution optimality as $W$ increases.
Importantly, \alb only improves the overall runtime (indicated by being below the 1.0 line in Figure~\ref{fig:subfig2}) for $W \in \{1.1, 1.2\}$ for \WBSstar and when $W=1.2$ for \bwa; in all other scenarios, it adversely affects performance.



\section{Summary and Conclusions}

We integrated the core concept of Weighted A* (\wastar) into \bae, showing how the choice of \(\lambda\), which weights the heuristic error \(d\), affects performance.
We proposed tighter lower bound methods and showed that with a well-chosen \(\lambda\), \bWae outperforms related algorithms in many cases. We also provided practical guidelines for using \bWae and selecting \(\lambda\), and conducted experiments with tuned \(\lambda\) values.

Future work could develop \bihs and \bss algorithms inspired by strategies beyond \wastar, such as DPS~\cite{gilon2016dynamic}, EES~\cite{Thayer2011BoundedSS}, and XDP/XUP~\cite{Chen2021NecessaryAS}. \bWae could also be extended to use separate \(\lambda\) values for forward and backward search, which may be  beneficial with asymmetric heuristics. 


\section*{Acknowledgements}
The work of Shahaf Shperberg and Ariel Felner was supported by the Israel Science Foundation (ISF)
grant \#909/23. This work was also supported by Israel's Ministry of Innovation, Science and Technology (MOST) grant \#1001706842, in collaboration with Israel National Road Safety Authority and Netivei Israel, awarded to Shahaf Shperberg, and by BSF grant \#2024614 awarded to Shahaf Shperberg. The work of Dor Atzmon was supported by the Israel Science Foundation (ISF)
grant \#1511/25. This work was also supported by Israel's Ministry of Innovation, Science and Technology (MOST) grant \#6908 (Czech-Israeli cooperative scientific research) awarded to Dor Atzmon.

\bibliography{refs}

\newpage
\makeatletter
\@ifundefined{isChecklistMainFile}{
  \newif\ifreproStandalone
  \reproStandalonetrue
}{
  \newif\ifreproStandalone
  \reproStandalonefalse
}
\makeatother

\ifreproStandalone
\documentclass[letterpaper]{article}
\usepackage[submission]{aaai2026}
\setlength{\pdfpagewidth}{8.5in}
\setlength{\pdfpageheight}{11in}
\usepackage{times}
\usepackage{helvet}
\usepackage{courier}
\usepackage{xcolor}
\frenchspacing

\begin{document}
\fi
\setlength{\leftmargini}{20pt}
\makeatletter\def\@listi{\leftmargin\leftmargini \topsep .5em \parsep .5em \itemsep .5em}
\def\@listii{\leftmargin\leftmarginii \labelwidth\leftmarginii \advance\labelwidth-\labelsep \topsep .4em \parsep .4em \itemsep .4em}
\def\@listiii{\leftmargin\leftmarginiii \labelwidth\leftmarginiii \advance\labelwidth-\labelsep \topsep .4em \parsep .4em \itemsep .4em}\makeatother

\setcounter{secnumdepth}{0}
\renewcommand\thesubsection{\arabic{subsection}}
\renewcommand\labelenumi{\thesubsection.\arabic{enumi}}

\newcounter{checksubsection}
\newcounter{checkitem}[checksubsection]

\newcommand{\checksubsection}[1]{%
  \refstepcounter{checksubsection}%
  \paragraph{\arabic{checksubsection}. #1}%
  \setcounter{checkitem}{0}%
}

\newcommand{\checkitem}{%
  \refstepcounter{checkitem}%
  \item[\arabic{checksubsection}.\arabic{checkitem}.]%
}
\newcommand{\question}[2]{\normalcolor\checkitem #1 #2 \color{blue}}
\newcommand{\ifyespoints}[1]{\makebox[0pt][l]{\hspace{-15pt}\normalcolor #1}}

\section*{Reproducibility Checklist}

\vspace{1em}
\hrule
\vspace{1em}








\checksubsection{General Paper Structure}
\begin{itemize}

\question{Includes a conceptual outline and/or pseudocode description of AI methods introduced}{(yes/partial/no/NA)}
yes

\question{Clearly delineates statements that are opinions, hypothesis, and speculation from objective facts and results}{(yes/no)}
yes

\question{Provides well-marked pedagogical references for less-familiar readers to gain background necessary to replicate the paper}{(yes/no)}
yes

\end{itemize}
\checksubsection{Theoretical Contributions}
\begin{itemize}

\question{Does this paper make theoretical contributions?}{(yes/no)}
yes

	\ifyespoints{\vspace{1.2em}If yes, please address the following points:}
        \begin{itemize}
	
	\question{All assumptions and restrictions are stated clearly and formally}{(yes/partial/no)}
	yes

	\question{All novel claims are stated formally (e.g., in theorem statements)}{(yes/partial/no)}
	yes

	\question{Proofs of all novel claims are included}{(yes/partial/no)}
	yes

	\question{Proof sketches or intuitions are given for complex and/or novel results}{(yes/partial/no)}
	yes

	\question{Appropriate citations to theoretical tools used are given}{(yes/partial/no)}
	yes

	\question{All theoretical claims are demonstrated empirically to hold}{(yes/partial/no/NA)}
	yes

	\question{All experimental code used to eliminate or disprove claims is included}{(yes/no/NA)}
	yes
	
	\end{itemize}
\end{itemize}

\checksubsection{Dataset Usage}
\begin{itemize}

\question{Does this paper rely on one or more datasets?}{(yes/no)}
yes

\ifyespoints{If yes, please address the following points:}
\begin{itemize}

	\question{A motivation is given for why the experiments are conducted on the selected datasets}{(yes/partial/no/NA)}
	yes

	\question{All novel datasets introduced in this paper are included in a data appendix}{(yes/partial/no/NA)}
	NA

	\question{All novel datasets introduced in this paper will be made publicly available upon publication of the paper with a license that allows free usage for research purposes}{(yes/partial/no/NA)}
	NA

	\question{All datasets drawn from the existing literature (potentially including authors' own previously published work) are accompanied by appropriate citations}{(yes/no/NA)}
	yes

	\question{All datasets drawn from the existing literature (potentially including authors' own previously published work) are publicly available}{(yes/partial/no/NA)}
	yes

	\question{All datasets that are not publicly available are described in detail, with explanation why publicly available alternatives are not scientifically satisficing}{(yes/partial/no/NA)}
	NA

\end{itemize}
\end{itemize}

\checksubsection{Computational Experiments}
\begin{itemize}

\question{Does this paper include computational experiments?}{(yes/no)}
yes

\ifyespoints{If yes, please address the following points:}
\begin{itemize}

	\question{This paper states the number and range of values tried per (hyper-) parameter during development of the paper, along with the criterion used for selecting the final parameter setting}{(yes/partial/no/NA)}
	yes

	\question{Any code required for pre-processing data is included in the appendix}{(yes/partial/no)}
	no

	\question{All source code required for conducting and analyzing the experiments is included in a code appendix}{(yes/partial/no)}
	no

	\question{All source code required for conducting and analyzing the experiments will be made publicly available upon publication of the paper with a license that allows free usage for research purposes}{(yes/partial/no)}
	yes
        
	\question{All source code implementing new methods have comments detailing the implementation, with references to the paper where each step comes from}{(yes/partial/no)}
	yes

	\question{If an algorithm depends on randomness, then the method used for setting seeds is described in a way sufficient to allow replication of results}{(yes/partial/no/NA)}
	NA

	\question{This paper specifies the computing infrastructure used for running experiments (hardware and software), including GPU/CPU models; amount of memory; operating system; names and versions of relevant software libraries and frameworks}{(yes/partial/no)}
	yes

	\question{This paper formally describes evaluation metrics used and explains the motivation for choosing these metrics}{(yes/partial/no)}
	yes

	\question{This paper states the number of algorithm runs used to compute each reported result}{(yes/no)}
	yes

	\question{Analysis of experiments goes beyond single-dimensional summaries of performance (e.g., average; median) to include measures of variation, confidence, or other distributional information}{(yes/no)}
	no

	\question{The significance of any improvement or decrease in performance is judged using appropriate statistical tests (e.g., Wilcoxon signed-rank)}{(yes/partial/no)}
	no

	\question{This paper lists all final (hyper-)parameters used for each model/algorithm in the paper’s experiments}{(yes/partial/no/NA)}
	yes

\end{itemize}
\end{itemize}
\ifreproStandalone
\end{document}
\fi
\newpage

\onecolumn

\appendix
\section*{\LARGE
Supplementary Materials for the ``Bidirectional Bounded-Suboptimal Heuristic Search with Consistent Heuristics" Paper}

\section*{Appendix A: Bounded suboptilmality of \bwa}
We prove below that \bwa is guaranteed to return bounded suboptimal algorithms. The proof is a direct generalization of the proof of \wastar.
\begin{theorem}
\bwa is bounded-suboptimal.
\end{theorem}

\begin{proof}

We need to prove that when a solution is returned with cost $U \leq LB_{W}$ then it holds that $LB_{W} \leq W \cdot C^*$.
Without loss of generality, assume that the $\max$ term is $\prWminf$.
Let $x$ be the node for which $\prWF(x)=\prWminf$ and let $y$ be the node for which $\fF(y)=\gF(y)+\hF(y)=\fminf$, where $\fminf$ is the minimal $f$-value in \openF.

Given an admissible heuristic, it holds that $\fF(y) \leq \Cstar$.
Note that the open lists are sorted according to $\prWD$ so $\prWF(x) \leq \prWF(y)$.
As a result, we get:
\begin{align*}
    U &\leq \prWminf \\
    &=\gF(x)+W \cdot \hF(x)\\
    &\leq \gF(y)+W \cdot \hF(y)\\
    &\leq W \cdot \gF(y)+W \cdot \hF(y)\\
    &= W \cdot \fF(y) 
    \leq W \cdot \Cstar
\end{align*}

Thus, \bwa returns solutions that are bounded by $W \cdot \Cstar$, making it a \bss algorithm.
\end{proof}

\section*{Appendix B: Absence of Node Re-expansion with Consistent Heuristics}
\setcounter{theorem}{2}
\begin{theorem}
Given consistent heuristics \bWae finds a bounded-suboptimal solution even if never re-expand nodes.
\end{theorem}
\begin{proof}
To prove this lemma, it suffices to show that every node is expanded with a bounded-suboptimal path. Specifically, we need to demonstrate two conditions:
\begin{itemize}
    \item[(i)] Whenever a node $u$ is expanded in the forward direction ($F$), it holds that $\gF(u) \leq W \cdot c(\start,u)$.
    \item[(ii)] Similarly, whenever a node $u$ is expanded in the backward direction ($B$), it holds that $\gB(u) \leq W \cdot c(u,\goal)$.
\end{itemize}

We will prove condition (i); the proof for condition (ii) is symmetrical.

Assume, for the sake of contradiction, that \bWae expands a node $u$ such that $\gF(u) > W \cdot c(\start,u)$. Further, assume that $u$ is the first node expanded in this manner. Note that $u \neq \start$, since the start node is expanded with an optimal cost of zero.

Given that $u$ is the first node expanded with an unbounded suboptimal path cost, there must exist another node $u' \in \openF$ that lies on the optimal path from $\start$ to $u$, satisfying $\gF(u') + c(u',u) \leq W \cdot c(\start,u)$. Since $u'$ is on the optimal path, we have $c(\start,u) = c(\start,u') + c(u',u)$. Now, because $\gF(u) > W \cdot c(\start,u)$, it follows that:
\[
\gF(u) > W \cdot c(\start,u') + W \cdot c(u',u).
\]
Since $\gF(u') \leq W \cdot c(\start,u')$, we have:
\[
\gF(u) > \gF(u') + W \cdot c(u',u).
\]
Additionally, since $c(u',u) \geq 0$, we know $\gF(u') < \gF(u)$.

Thus, we can evaluate the $\bF$ values:
\[
\bWF(u') = \gF(u') + W\cdot \hF(u') + \lambda (\gF(u')-\hB(u')).
\]
Substituting the known inequalities, we get:
\begin{align*}
\bWF(u') < &\gF(u) - W \cdot c(u',u)+  W\cdot \hF(u') \\&+ \lambda (\gF(u) - W \cdot c(u',u) -\hB(u')).
\end{align*}
Using the consistency of the heuristics, $\hF(u') \leq c(u',u) + \hF(u)$ and $\hB(u) \leq c(u',u) + \hB(u')$, we obtain:
\begin{align*}
\bWF(u') < &\gF(u)  + W\cdot \hF(u) \\&+ \lambda (\gF(u) - \hB(u) - (W - 1)c(u',u)).
\end{align*}
Since $W \geq 1$ and $c(u',u) \geq 0$, this simplifies to:
\[
\bWF(u') < \gF(u) + W\cdot \hF(u) + \lambda (\gF(u) - \hB(u)) = \bWF(u).
\]
This contradicts the assumption that $u$ is expanded before $u'$.

Consequently, every path discovered during the search is guaranteed to be bounded-suboptimal. Since the algorithm terminates only once a path is found whose cost is greater or equal to the lower bound, it must return a path of cost that does not exceed $W\cdot\Cstar$.
\end{proof}

\section*{Appendix C: Runetime and Solution Quality Results}
This section includes the full runtime and solution quality results of our experiments. Tables 5-9 report the average runtime for each domain and algorithm, whereas tables 10-14 show the quality of solutions found by each algorithm.

Experiments were conducted on a 92-machine cluster, each equipped with an AMD EPYC 7763 CPU (256 hyperthreads) and 1000GB of RAM. Each combination of search algorithm and problem instance was run on a single thread with a 60GB memory limit. While runtime measurements may be somewhat noisy due to the distributed setup, all algorithms follow a similar operational pattern: computing lower bounds, selecting nodes via a priority function, expanding nodes, and maintaining the incumbent solution. These operations incur comparable overhead across algorithms, as confirmed by the code, with the exception of \wmm, which is inherently slower due to its lower-bound computation that involves iterating over buckets in the open list. Therefore, although we report runtime results, we regard node expansions as a more reliable indicator of performance.

\begin{table}[ht]
\caption{Average runtime in seconds on STP domain.}
\centering
\begin{tabular}{lrrrrrrrrr}
\toprule
\multicolumn{1}{c}{Algorithm} & \multicolumn{1}{c}{1} & \multicolumn{1}{c}{1.1} & \multicolumn{1}{c}{1.2} & \multicolumn{1}{c}{1.5} & \multicolumn{1}{c}{1.7} & \multicolumn{1}{c}{2} & \multicolumn{1}{c}{3} & \multicolumn{1}{c}{5} & \multicolumn{1}{c}{10} \\
\midrule
\wastar & 59.28 & 33.29 & 10.36 & 0.80 & 0.22 & 0.08 & \textbf{0.02} & \textbf{0.01} & \textbf{0.01} \\
\bwa & 71.31 & 31.22 & 2.50 & \textbf{0.20} & \textbf{0.12} & \textbf{0.07} & \textbf{0.02} & \textbf{0.01} & \textbf{0.01} \\
\wmm & 73.55 & 14.94 & 3.06 & 0.46 & 0.51 & 0.82 & 9.33 & 11.33 & 0.12 \\
\WBSstar & 45.54 & 9.52 & \textbf{1.39} & 0.21 & 0.14 & 0.08 & \textbf{0.02} & \textbf{0.01} & \textbf{0.01} \\
\bWae$\frac{1}{W^2}$ & 12.45 & \textbf{3.69} & 1.55 & 0.43 & 0.20 & \textbf{0.07} & 0.03 & \textbf{0.01} & \textbf{0.01} \\
\bWae$\frac{1}{W}$ & 12.35 & 4.07 & 2.14 & 0.65 & 0.41 & 0.14 & 0.03 & \textbf{0.01} & \textbf{0.01} \\
\bWae$1$ & \textbf{11.56} & 4.17 & 2.92 & 1.19 & 0.85 & 0.53 & 0.09 & 0.02 & \textbf{0.01} \\
\bWae$W$ & 13.06 & 4.36 & 3.75 & 2.75 & 2.36 & 2.16 & 4.16 & 3.57 & 2.12 \\
\bWae$\lambda^*$ & 14.88 & 5.00 & 1.56 & 0.40 & 0.16 & 0.13 & \textbf{0.02} & \textbf{0.01} & \textbf{0.01} \\

\bottomrule
\end{tabular}
\end{table}

\begin{table}[ht]
\caption{Average runtime in milliseconds on DAO domain.}
\centering
\begin{tabular}{lrrrrrrrrr}
\toprule
\multicolumn{1}{c}{Algorithm} & \multicolumn{1}{c}{1} & \multicolumn{1}{c}{1.1} & \multicolumn{1}{c}{1.2} & \multicolumn{1}{c}{1.5} & \multicolumn{1}{c}{1.7} & \multicolumn{1}{c}{2} & \multicolumn{1}{c}{3} & \multicolumn{1}{c}{5} & \multicolumn{1}{c}{10} \\
\midrule
\wastar & \textbf{2.80} & \textbf{2.66} & \textbf{2.62} & \textbf{2.12} & \textbf{1.90} & \textbf{1.67} & \textbf{1.41} & \textbf{1.25} & \textbf{1.14} \\
\bwa & 4.10 & 3.66 & 3.19 & 2.76 & 2.52 & 2.21 & 1.88 & 1.67 & 1.50 \\
\wmm & 504.99 & 447.72 & 190.39 & 116.72 & 78.09 & 91.19 & 13.76 & 3.82 & 3.22 \\
\WBSstar & 3.08 & 2.85 & 2.60 & 2.13 & 1.97 & 1.72 & 1.56 & 1.44 & 1.37 \\
\bWae$\frac{1}{W^2}$ & 3.96 & 3.81 & 3.59 & 2.89 & 2.52 & 2.31 & 1.97 & 1.72 & 1.54 \\
\bWae$\frac{1}{W}$ & 4.22 & 4.17 & 3.72 & 3.10 & 2.78 & 2.20 & 1.85 & 1.60 & 1.42 \\
\bWae$1$ & 3.84 & 3.85 & 3.67 & 3.17 & 2.92 & 2.60 & 2.10 & 1.78 & 1.56 \\
\bWae$W$ & 3.86 & 3.94 & 3.94 & 3.46 & 3.38 & 3.15 & 2.89 & 2.80 & 2.65 \\
\bWae$\lambda^*$ & 3.90 & 3.20 & 3.18 & 2.76 & 2.19 & 1.96 & 1.71 & 1.68 & 1.41 \\

\bottomrule
\end{tabular}
\end{table}

\begin{table}[ht]
\caption{Average runtime in milliseconds on Mazes domain.}
\centering
\begin{tabular}{lrrrrrrrrr}
\toprule
\multicolumn{1}{c}{Algorithm} & \multicolumn{1}{c}{1} & \multicolumn{1}{c}{1.1} & \multicolumn{1}{c}{1.2} & \multicolumn{1}{c}{1.5} & \multicolumn{1}{c}{1.7} & \multicolumn{1}{c}{2} & \multicolumn{1}{c}{3} & \multicolumn{1}{c}{5} & \multicolumn{1}{c}{10} \\
\midrule
\wastar & 47.45 & 49.59 & 49.21 & \textbf{47.04} & \textbf{46.71} & \textbf{44.47} & \textbf{39.89} & \textbf{34.59} & \textbf{29.87} \\
\bwa & 63.31 & 66.27 & 66.99 & 60.98 & 60.10 & 55.65 & 49.30 & 43.74 & 38.75 \\
\wmm & 1,649 & 2,334 & 3,033 & 7,683 & 12,302 & 17,147 & 19,502 & 10,452 & 289.27 \\
\WBSstar & \textbf{45.08} & \textbf{48.83} & \textbf{49.01} & \textbf{47.04} & 46.36 & 39.38 & 35.81 & 32.03 & 30.08 \\
\bWae$\frac{1}{W^2}$ & 55.25 & 58.88 & 60.95 & 63.45 & 62.79 & 60.86 & 54.83 & 46.69 & 41.52 \\
\bWae$\frac{1}{W}$ & 46.81 & 50.20 & 58.42 & 59.96 & 59.81 & 58.58 & 54.18 & 47.63 & 41.87 \\
\bWae$1$ & 48.52 & 50.10 & 49.30 & 48.43 & 52.39 & 50.15 & 47.84 & 45.79 & 42.34 \\
\bWae$W$ & 49.87 & 50.95 & 51.62 & 53.20 & 54.12 & 52.22 & 51.83 & 52.75 & 55.47 \\
\bWae$\lambda^*$ & 46.24 & 49.82 & 83.40 & 51.24 & 51.02 & 50.33 & 42.78 & 41.27 & 40.73 \\

\bottomrule
\end{tabular}
\end{table}

\begin{table}[ht]
\caption{Average runtime in milliseconds on 18 Pancake domain, with GAP to GAP-2 heuristics.}
\centering
\begin{adjustbox}
{width=0.99\textwidth,center}
\begin{tabular}{@{}l|rrrrrrrrr|rrrrrrrrr|rrrrrrrrr@{}}
\toprule
\multicolumn{1}{c}{} & \multicolumn{9}{c}{GAP} & \multicolumn{9}{c}{GAP-1} & \multicolumn{9}{c}{GAP-2} \\ \midrule
\multicolumn{1}{c|}{Algorithm} & \multicolumn{1}{c}{1} & \multicolumn{1}{c}{1.1} & \multicolumn{1}{c}{1.2} & \multicolumn{1}{c}{1.5} & \multicolumn{1}{c}{1.7} & \multicolumn{1}{c}{2} & \multicolumn{1}{c}{3} & \multicolumn{1}{c}{5} & \multicolumn{1}{c|}{10} & \multicolumn{1}{c}{1} & \multicolumn{1}{c}{1.1} & \multicolumn{1}{c}{1.2} & \multicolumn{1}{c}{1.5} & \multicolumn{1}{c}{1.7} & \multicolumn{1}{c}{2} & \multicolumn{1}{c}{3} & \multicolumn{1}{c}{5} & \multicolumn{1}{c|}{10} & \multicolumn{1}{c}{1} & \multicolumn{1}{c}{1.1} & \multicolumn{1}{c}{1.2} & \multicolumn{1}{c}{1.5} & \multicolumn{1}{c}{1.7} & \multicolumn{1}{c}{2} & \multicolumn{1}{c}{3} & \multicolumn{1}{c}{5} & \multicolumn{1}{c}{10} \\
\cmidrule(lr){2-10} \cmidrule(lr){11-19} \cmidrule(lr){20-28}
\wastar & 3.57 & 2.75 & 0.91 & \textbf{0.56} & \textbf{0.54} & \textbf{0.46} & \textbf{0.52} & \textbf{0.50} & \textbf{0.50} & 1,013 & 602 & 63.21 & 4.27 & 2.20 & 0.92 & 0.69 & 0.64 & 0.64 & 167,115 & 145,095 & 30,996 & 1,688 & 411 & 68.98 & 20.45 & 17.06 & 16.74 \\
\bwa & 2.58 & \textbf{1.39} & \textbf{0.87} & 0.65 & 0.62 & 0.58 & 0.58 & 0.64 & 0.60 & 715 & 57.04 & 10.50 & \textbf{0.71} & \textbf{0.59} & \textbf{0.52} & \textbf{0.51} & \textbf{0.51} & \textbf{0.51} & 134,641 & 30,062 & 3,708 & 70.74 & 17.86 & 4.71 & \textbf{2.34} & \textbf{2.14} & \textbf{2.20} \\
\wmm & 5.14 & 3.23 & 4.22 & 25.38 & 81.93 & 394 & 99.29 & 0.79 & 0.58 & 671 & 55.45 & 15.59 & 26.90 & 63.75 & 174.57 & 2,204 & 9.83 & 0.65 & 67,578 & 11,739 & 1,782 & 155 & 74.26 & 279 & 9,574 & 963 & 5.72 \\
\WBSstar & 2.85 & 1.83 & 1.11 & 0.79 & 0.76 & 0.73 & 0.73 & 0.72 & 0.72 & 316 & 36.19 & 8.58 & 0.77 & 0.73 & 0.66 & 0.65 & 0.64 & 0.64 & 59,861 & 16,464 & 1,735 & 47.65 & \textbf{8.66} & \textbf{4.14} & 2.82 & 2.64 & 2.73 \\
\bWae$\frac{1}{W^2}$ & 2.73 & 1.72 & 1.27 & 0.73 & 0.71 & 0.64 & 0.63 & 0.63 & 0.63 & 29.11 & \textbf{8.92} & 5.97 & 1.06 & 0.82 & 0.67 & 0.59 & 0.58 & 0.58 & \textbf{2,681} & 913 & 143 & 51.55 & 24.20 & 8.48 & 2.64 & 2.62 & 2.70 \\
\bWae$\frac{1}{W}$ & \textbf{2.71} & 1.74 & 1.30 & 0.76 & 0.72 & 0.63 & 0.63 & 0.62 & 0.63 & 29.59 & 9.61 & 8.53 & 1.77 & 1.01 & 0.68 & 0.59 & 0.57 & 0.57 & 2,688 & \textbf{725} & 140 & 145 & 35.25 & 13.31 & 2.59 & 2.58 & 2.67 \\
\bWae$1$ & 2.73 & 1.80 & 1.31 & 0.81 & 0.74 & 0.67 & 0.64 & 0.63 & 0.64 & 29.31 & 10.67 & 9.62 & 2.49 & 1.74 & 0.86 & 0.62 & 0.58 & 0.58 & 2,760 & 739 & 199 & 268 & 196 & 36.53 & 5.28 & 2.96 & 2.87 \\
\bWae$W$ & 2.72 & 1.77 & 1.35 & 0.95 & 0.91 & 0.74 & 0.79 & 0.93 & 0.78 & 30.65 & 11.68 & 15.56 & 34.36 & 24.15 & 15.48 & 6.11 & 3.38 & 2.67 & 3,379 & 918 & 409 & 1,716 & 3,229 & 11,604 & 2,023 & 1,038 & 586 \\
\bWae$\lambda^*$ & 2.90 & 1.75 & 1.03 & 0.80 & 0.67 & 0.61 & 0.59 & 0.63 & 0.61 & \textbf{26.95} & 10.18 & \textbf{4.31} & 1.33 & 0.96 & 0.82 & 0.84 & 0.79 & 0.69 & 2,776 & 877 & \textbf{123} & \textbf{45.63} & 17.26 & 8.64 & 3.00 & 2.84 & 3.09 \\

\bottomrule
\end{tabular}
\end{adjustbox}
\end{table}

\begin{table}[ht]
\caption{Average runtime in milliseconds on 12-disks ToH domain, with $(10+2)$, $(8+4)$, $(6+6)$ PDBS.}
\centering
\begin{adjustbox}{width=0.99\textwidth,center}
\begin{tabular}{l|rrrrrrrrr|rrrrrrrrr|rrrrrrrrr}
\toprule
\multicolumn{1}{c}{} & \multicolumn{9}{c}{ToH-12 (10+2)} & \multicolumn{9}{c}{ToH-12 (8+4)} & \multicolumn{9}{c}{ToH-12 (6+6)} \\ \midrule
\multicolumn{1}{c|}{Algorithm} & \multicolumn{1}{c}{1} & \multicolumn{1}{c}{1.1} & \multicolumn{1}{c}{1.2} & \multicolumn{1}{c}{1.5} & \multicolumn{1}{c}{1.7} & \multicolumn{1}{c}{2} & \multicolumn{1}{c}{3} & \multicolumn{1}{c}{5} & \multicolumn{1}{c|}{10} & \multicolumn{1}{c}{1} & \multicolumn{1}{c}{1.1} & \multicolumn{1}{c}{1.2} & \multicolumn{1}{c}{1.5} & \multicolumn{1}{c}{1.7} & \multicolumn{1}{c}{2} & \multicolumn{1}{c}{3} & \multicolumn{1}{c}{5} & \multicolumn{1}{c|}{10} & \multicolumn{1}{c}{1} & \multicolumn{1}{c}{1.1} & \multicolumn{1}{c}{1.2} & \multicolumn{1}{c}{1.5} & \multicolumn{1}{c}{1.7} & \multicolumn{1}{c}{2} & \multicolumn{1}{c}{3} & \multicolumn{1}{c}{5} & \multicolumn{1}{c}{10} \\
\cmidrule(lr){2-10} \cmidrule(lr){11-19} \cmidrule(lr){20-28}
\wastar & 518 & 245 & 86.28 & 5.33 & 2.78 & 1.58 & \textbf{0.54} & 0.56 & 0.57 & 4,658 & 3,641 & 2,429 & 995 & 628 & 521 & 30.38 & 13.16 & 19.09 & 8,136 & 7,641 & 6,502 & 3,405 & 2,994 & 2,350 & 2,128 & 1,182 & 326 \\
\bwa & 534 & 226 & 72.69 & 3.65 & 2.07 & 0.95 & 0.59 & 0.55 & 0.59 & 3,343 & 2,346 & 1,533 & 424 & 227 & 63.35 & \textbf{15.24} & \textbf{7.29} & \textbf{5.16} & 6,025 & 4,604 & 3,547 & 1,864 & 1,108 & 527 & \textbf{142} & \textbf{81} & 47.50 \\
\wmm & 793 & 358 & 112 & 5.16 & 3.30 & 2.28 & 6.52 & 5.40 & 0.84 & 2,406 & 2,122 & 1,712 & 611 & 342 & 94.53 & 22.37 & 12.66 & 10.44 & 2,142 & 2,428 & 2,526 & 2,110 & 1,473 & 885 & 508 & 202 & 93.06 \\
\WBSstar & 334 & 137 & 45.52 & \textbf{3.41} & \textbf{1.86} & \textbf{0.92} & 0.55 & \textbf{0.51} & \textbf{0.56} & 2,327 & 1,670 & 1,142 & 337 & 172 & \textbf{49.96} & 15.36 & 7.74 & 5.70 & 4,227 & 3,422 & 2,736 & 1,391 & 902 & 491 & 193 & 114 & 64.83 \\
\bWae$\frac{1}{W^2}$ & \textbf{75.68} & \textbf{47.66} & \textbf{27.01} & 8.85 & 4.01 & 1.80 & 0.66 & 0.60 & 0.68 & \textbf{365} & 398 & 328 & \textbf{201} & \textbf{130} & 55.03 & 18.42 & 7.94 & 5.84 & \textbf{812} & 966 & 947 & 833 & \textbf{723} & \textbf{453} & 206 & 110 & \textbf{64} \\
\bWae$\frac{1}{W}$ & 94.91 & 53.87 & 32.47 & 11.75 & 6.23 & 2.60 & 0.69 & 0.64 & 0.63 & 494 & 444 & 368 & 239 & 180 & 85.51 & 23.97 & 9.28 & 5.84 & 1,051 & 1,076 & 1,015 & 870 & 809 & 583 & 289 & 147 & 74.67 \\
\bWae$1$ & 97.12 & 53.70 & 37.08 & 16.73 & 10.70 & 4.94 & 0.93 & 0.69 & 0.66 & 434 & 371 & 325 & 255 & 221 & 152 & 41.28 & 14.22 & 6.74 & 950 & 921 & 885 & 822 & 799 & 682 & 413 & 233 & 119 \\
\bWae$W$ & 92.77 & 52.04 & 39.71 & 22.73 & 18.72 & 11.94 & 5.44 & 2.72 & 1.98 & 485 & 351 & 321 & 260 & 241 & 203 & 159 & 123 & 81.30 & 1,059 & \textbf{864} & \textbf{838} & 763 & 743 & 663 & 584 & 510 & 400 \\
\bWae$\lambda^*$ & 129 & 48.39 & 30.67 & 7.31 & 2.84 & 1.64 & 0.86 & 0.76 & 0.77 & 611 & \textbf{281} & \textbf{264} & 214 & 142 & 124 & 19.16 & 9.71 & 6.67 & 1,120 & 1,751 & 1,147 & \textbf{704} & 997 & 769 & 250 & 145 & 124 \\

\bottomrule
\end{tabular}
\end{adjustbox}
\end{table}

\begin{table}[ht]
\caption{Average solution quality on STP domain.}
\centering
\begin{tabular}{lrrrrrrrrr}
\toprule
\multicolumn{1}{c}{Algorithm} & \multicolumn{1}{c}{1} & \multicolumn{1}{c}{1.1} & \multicolumn{1}{c}{1.2} & \multicolumn{1}{c}{1.5} & \multicolumn{1}{c}{1.7} & \multicolumn{1}{c}{2} & \multicolumn{1}{c}{3} & \multicolumn{1}{c}{5} & \multicolumn{1}{c}{10} \\
\midrule
\wastar & 1.000 & \textbf{1.001} & \textbf{1.011} & 1.063 & 1.103 & 1.199 & 1.468 & 1.848 & 2.242 \\
\bwa & 1.000 & 1.002 & 1.015 & 1.100 & 1.140 & 1.221 & 1.468 & 1.806 & 2.165 \\
\wmm & 1.000 & 1.002 & 1.017 & 1.106 & 1.182 & 1.298 & 1.616 & 1.489 & 1.892 \\
\WBSstar & 1.000 & 1.003 & 1.021 & 1.102 & 1.140 & 1.223 & 1.469 & 1.806 & 2.165 \\
\bWae$\frac{1}{W^2}$ & 1.000 & 1.005 & 1.020 & 1.064 & 1.091 & 1.153 & 1.393 & 1.714 & 2.027 \\
\bWae$\frac{1}{W}$ & 1.000 & 1.006 & 1.017 & 1.051 & 1.069 & 1.122 & 1.325 & 1.668 & 2.039 \\
\bWae$1$ & 1.000 & 1.006 & 1.015 & 1.032 & 1.042 & 1.070 & 1.210 & 1.425 & 1.787 \\
\bWae$W$ & 1.000 & 1.009 & 1.014 & \textbf{1.017} & \textbf{1.023} & \textbf{1.030} & \textbf{1.067} & \textbf{1.106} & \textbf{1.139} \\
\bWae$\lambda^*$ & 1.000 & 1.003 & 1.017 & 1.098 & 1.132 & 1.181 & 1.419 & 1.714 & 2.035 \\

\bottomrule
\end{tabular}
\end{table}

\begin{table}[ht]
\caption{Average solution quality on DAO domain.}
\centering
\begin{tabular}{lrrrrrrrrr}
\toprule
\multicolumn{1}{c}{Algorithm} & \multicolumn{1}{c}{1} & \multicolumn{1}{c}{1.1} & \multicolumn{1}{c}{1.2} & \multicolumn{1}{c}{1.5} & \multicolumn{1}{c}{1.7} & \multicolumn{1}{c}{2} & \multicolumn{1}{c}{3} & \multicolumn{1}{c}{5} & \multicolumn{1}{c}{10} \\
\midrule
\wastar & 1.000 & \textbf{1.006} & \textbf{1.012} & \textbf{1.025} & \textbf{1.026} & \textbf{1.034} & \textbf{1.044} & 1.065 & 1.077 \\
\bwa & 1.000 & 1.013 & 1.025 & 1.045 & 1.047 & 1.055 & 1.067 & 1.089 & 1.102 \\
\wmm & 1.000 & 1.019 & 1.038 & 1.071 & 1.075 & 1.090 & 1.094 & 1.109 & 1.114 \\
\WBSstar & 1.000 & 1.024 & 1.037 & 1.054 & 1.056 & 1.071 & 1.083 & 1.108 & 1.124 \\
\bWae$\frac{1}{W^2}$ & 1.000 & 1.009 & 1.017 & 1.034 & 1.045 & 1.054 & 1.070 & 1.092 & 1.108 \\
\bWae$\frac{1}{W}$ & 1.000 & 1.009 & 1.017 & 1.032 & 1.037 & 1.050 & 1.069 & 1.091 & 1.107 \\
\bWae$1$ & 1.000 & 1.010 & 1.017 & 1.030 & 1.034 & 1.042 & 1.061 & 1.079 & 1.095 \\
\bWae$W$ & 1.000 & 1.010 & 1.016 & 1.028 & 1.031 & 1.039 & \textbf{1.044} & \textbf{1.050} & \textbf{1.055} \\
\bWae $\lambda$ & 1.000 & 1.011 & 1.022 & 1.043 & 1.049 & 1.056 & 1.070 & 1.092 & 1.108 \\

\bottomrule
\end{tabular}
\end{table}

\begin{table}[ht]
\caption{Average solution quality on Mazes domain.}
\centering
\begin{tabular}{lrrrrrrrrr}
\toprule
\multicolumn{1}{c}{Algorithm} & \multicolumn{1}{c}{1} & \multicolumn{1}{c}{1.1} & \multicolumn{1}{c}{1.2} & \multicolumn{1}{c}{1.5} & \multicolumn{1}{c}{1.7} & \multicolumn{1}{c}{2} & \multicolumn{1}{c}{3} & \multicolumn{1}{c}{5} & \multicolumn{1}{c}{10} \\
\midrule
\wastar & 1.000 & 1.002 & 1.006 & 1.013 & 1.013 & 1.020 & 1.030 & 1.064 & 1.082 \\
\bwa & 1.000 & 1.002 & 1.006 & 1.016 & 1.017 & 1.024 & 1.036 & 1.070 & 1.089 \\
\wmm & 1.000 & 1.001 & 1.005 & 1.013 & 1.014 & 1.019 & 1.031 & 1.065 & 1.088 \\
\WBSstar & 1.000 & 1.003 & 1.008 & 1.018 & 1.019 & 1.027 & 1.039 & 1.074 & 1.093 \\
\bWae$\frac{1}{W^2}$ & 1.000 & \textbf{1.000} & \textbf{1.001} & 1.009 & 1.013 & 1.019 & 1.034 & 1.068 & 1.091 \\
\bWae$\frac{1}{W}$ & 1.000 & \textbf{1.000} & \textbf{1.001} & 1.008 & 1.010 & 1.015 & 1.032 & 1.065 & 1.091 \\
\bWae$1$ & 1.000 & 1.001 & \textbf{1.001} & \textbf{1.007} & 1.009 & 1.011 & 1.023 & 1.043 & 1.069 \\
\bWae$W$ & 1.000 & 1.001 & \textbf{1.001} & \textbf{1.007} & \textbf{1.008} & \textbf{1.009} & \textbf{1.011} & \textbf{1.015} & \textbf{1.018} \\
\bWae$\lambda^*$ & 1.000 & 1.001 & \textbf{1.001} & 1.008 & 1.009 & 1.011 & 1.025 & 1.052 & 1.091 \\

\bottomrule
\end{tabular}
\end{table}

\begin{table}[t]
\caption{Average solution quality on 18 Pancakes domain, with GAP to GAP-2 heuristics.}
\centering
\begin{adjustbox}{width=0.95\textwidth,center}
\begin{tabular}{@{}l|rrrrrrrrr|rrrrrrrrr|rrrrrrrrr@{}}
\toprule
\multicolumn{1}{c}{} & \multicolumn{9}{c}{GAP} & \multicolumn{9}{c}{GAP-1} & \multicolumn{9}{c}{GAP-2} \\ \midrule
\multicolumn{1}{c|}{Algorithm} & \multicolumn{1}{c}{1} & \multicolumn{1}{c}{1.1} & \multicolumn{1}{c}{1.2} & \multicolumn{1}{c}{1.5} & \multicolumn{1}{c}{1.7} & \multicolumn{1}{c}{2} & \multicolumn{1}{c}{3} & \multicolumn{1}{c}{5} & \multicolumn{1}{c|}{10} & \multicolumn{1}{c}{1} & \multicolumn{1}{c}{1.1} & \multicolumn{1}{c}{1.2} & \multicolumn{1}{c}{1.5} & \multicolumn{1}{c}{1.7} & \multicolumn{1}{c}{2} & \multicolumn{1}{c}{3} & \multicolumn{1}{c}{5} & \multicolumn{1}{c|}{10} & \multicolumn{1}{c}{1} & \multicolumn{1}{c}{1.1} & \multicolumn{1}{c}{1.2} & \multicolumn{1}{c}{1.5} & \multicolumn{1}{c}{1.7} & \multicolumn{1}{c}{2} & \multicolumn{1}{c}{3} & \multicolumn{1}{c}{5} & \multicolumn{1}{c}{10} \\
\cmidrule(lr){2-10} \cmidrule(lr){11-19} \cmidrule(lr){20-28}
\wastar & 1.000 & \textbf{1.014} & 1.043 & 1.098 & 1.105 & 1.141 & 1.150 & 1.150 & 1.150 & 1.000 & 1.002 & \textbf{1.020} & 1.082 & 1.094 & 1.150 & 1.185 & 1.199 & 1.199 & 1.000 & 1.002 & \textbf{1.011} & \textbf{1.058} & 1.081 & 1.126 & 1.171 & 1.208 & 1.218 \\
\bwa & 1.000 & 1.017 & 1.039 & 1.073 & 1.077 & 1.103 & 1.103 & 1.103 & 1.103 & 1.000 & 1.002 & 1.029 & 1.082 & 1.093 & 1.134 & 1.145 & 1.146 & 1.146 & 1.000 & 1.002 & \textbf{1.011} & 1.071 & 1.093 & 1.134 & 1.178 & 1.202 & 1.198 \\
\wmm & 1.000 & 1.019 & 1.073 & 1.222 & 1.252 & 1.297 & 1.108 & 1.104 & 1.103 & 1.000 & \textbf{1.000} & 1.023 & 1.145 & 1.205 & 1.338 & 1.198 & \textbf{1.131} & \textbf{1.141} & 1.000 & \textbf{1.001} & 1.013 & 1.088 & 1.159 & 1.285 & 1.454 & 1.162 & 1.169 \\
\WBSstar & 1.000 & 1.017 & 1.040 & 1.075 & 1.077 & 1.103 & 1.103 & 1.103 & 1.103 & 1.000 & 1.002 & 1.031 & 1.087 & \textbf{1.095} & 1.138 & 1.145 & 1.146 & 1.146 & 1.000 & 1.002 & 1.012 & 1.097 & 1.120 & 1.161 & 1.182 & 1.202 & 1.198 \\
\bWae$\frac{1}{W^2}$ & 1.000 & 1.017 & \textbf{1.036} & 1.064 & 1.067 & 1.087 & 1.094 & \textbf{1.095} & \textbf{1.095} & 1.000 & 1.005 & 1.024 & 1.075 & 1.090 & 1.121 & \textbf{1.142} & 1.152 & 1.152 & 1.000 & 1.002 & 1.015 & 1.070 & 1.093 & 1.132 & 1.171 & 1.193 & 1.199 \\
\bWae$\frac{1}{W}$ & 1.000 & 1.017 & 1.039 & 1.060 & 1.066 & 1.087 & 1.094 & \textbf{1.095} & \textbf{1.095} & 1.000 & 1.006 & 1.025 & 1.066 & 1.084 & \textbf{1.116} & \textbf{1.142} & 1.152 & 1.152 & 1.000 & 1.003 & 1.017 & 1.069 & 1.080 & 1.122 & 1.170 & 1.193 & 1.199 \\
\bWae$1$ & 1.000 & 1.017 & 1.039 & 1.063 & 1.067 & 1.087 & 1.094 & 1.099 & 1.099 & 1.000 & 1.007 & 1.026 & \textbf{1.065} & \textbf{1.079} & 1.114 & 1.143 & 1.150 & 1.154 & 1.000 & 1.004 & 1.019 & 1.073 & \textbf{1.077} & 1.107 & 1.150 & 1.183 & 1.205 \\
\bWae$W$ & 1.000 & 1.017 & 1.039 & \textbf{1.059} & \textbf{1.062} & \textbf{1.082} & \textbf{1.090} & \textbf{1.095} & 1.097 & 1.000 & 1.012 & 1.037 & 1.094 & 1.106 & 1.126 & 1.158 & 1.173 & 1.180 & 1.000 & 1.009 & 1.026 & 1.063 & 1.084 & \textbf{1.101} & \textbf{1.134} & \textbf{1.156} & \textbf{1.172} \\
\bWae$\lambda^*$ & 1.000 & 1.017 & 1.037 & 1.065 & 1.071 & 1.087 & 1.094 & 1.098 & 1.099 & 1.000 & 1.005 & 1.026 & 1.075 & 1.091 & \textbf{1.116} & \textbf{1.142} & 1.152 & 1.152 & 1.000 & 1.008 & 1.017 & 1.069 & 1.094 & 1.132 & 1.163 & 1.188 & 1.184 \\

\bottomrule
\end{tabular}
\end{adjustbox}
\end{table}

\begin{table}[t]
\caption{Average solution quality on 12-disks ToH domain, with $(10+2)$, $(8+4)$, and $(6+6)$ PDBs.}
\centering
\begin{adjustbox}{width=0.95\textwidth,center}
\begin{tabular}{@{}l|rrrrrrrrr|rrrrrrrrr|rrrrrrrrr@{}}
\toprule
\multicolumn{1}{c}{} & \multicolumn{9}{c}{ToH-12 (10+2)} & \multicolumn{9}{c}{ToH-12 (8+4)} & \multicolumn{9}{c}{ToH-12 (6+6)} \\ \midrule
\multicolumn{1}{c|}{Algorithm} & \multicolumn{1}{c}{1} & \multicolumn{1}{c}{1.1} & \multicolumn{1}{c}{1.2} & \multicolumn{1}{c}{1.5} & \multicolumn{1}{c}{1.7} & \multicolumn{1}{c}{2} & \multicolumn{1}{c}{3} & \multicolumn{1}{c}{5} & \multicolumn{1}{c|}{10} & \multicolumn{1}{c}{1} & \multicolumn{1}{c}{1.1} & \multicolumn{1}{c}{1.2} & \multicolumn{1}{c}{1.5} & \multicolumn{1}{c}{1.7} & \multicolumn{1}{c}{2} & \multicolumn{1}{c}{3} & \multicolumn{1}{c}{5} & \multicolumn{1}{c|}{10} & \multicolumn{1}{c}{1} & \multicolumn{1}{c}{1.1} & \multicolumn{1}{c}{1.2} & \multicolumn{1}{c}{1.5} & \multicolumn{1}{c}{1.7} & \multicolumn{1}{c}{2} & \multicolumn{1}{c}{3} & \multicolumn{1}{c}{5} & \multicolumn{1}{c}{10} \\
\cmidrule(lr){2-10} \cmidrule(lr){11-19} \cmidrule(lr){20-28}
\wastar & 1.000 & \textbf{1.000} & \textbf{1.001} & 1.025 & 1.044 & 1.247 & 1.444 & 1.503 & 1.508 & 1.000 & \textbf{1.000} & \textbf{1.000} & 1.014 & 1.045 & 1.165 & 1.285 & 1.501 & 1.954 & 1.000 & \textbf{1.000} & \textbf{1.001} & 1.010 & 1.014 & 1.074 & 1.249 & 1.510 & 1.881 \\
\bwa & 1.000 & \textbf{1.000} & 1.002 & 1.051 & 1.106 & 1.248 & 1.431 & 1.542 & 1.560 & 1.000 & \textbf{1.000} & 1.001 & 1.027 & 1.075 & 1.159 & 1.378 & 1.711 & 2.057 & 1.000 & \textbf{1.000} & 1.005 & 1.039 & 1.063 & 1.129 & 1.365 & 1.611 & 2.017 \\
\wmm & 1.000 & \textbf{1.000} & \textbf{1.001} & 1.032 & 1.078 & 1.235 & 1.438 & 1.551 & 1.543 & 1.000 & \textbf{1.000} & \textbf{1.000} & 1.020 & 1.055 & 1.161 & 1.352 & 1.653 & 2.043 & 1.000 & \textbf{1.000} & \textbf{1.001} & 1.031 & 1.055 & 1.134 & 1.362 & 1.532 & 1.937 \\
\WBSstar & 1.000 & \textbf{1.000} & 1.003 & 1.050 & 1.102 & 1.246 & 1.435 & 1.544 & 1.563 & 1.000 & \textbf{1.000} & 1.002 & 1.041 & 1.082 & 1.164 & 1.375 & 1.725 & 2.061 & 1.000 & 1.001 & 1.008 & 1.048 & 1.072 & 1.140 & 1.380 & 1.623 & 2.033 \\
\bWae$\frac{1}{W^2}$ & 1.000 & 1.004 & 1.005 & 1.027 & 1.050 & 1.135 & 1.333 & 1.508 & 1.527 & 1.000 & \textbf{1.000} & 1.004 & 1.032 & 1.073 & 1.150 & 1.309 & 1.664 & 2.027 & 1.000 & 1.001 & 1.003 & 1.031 & 1.063 & 1.126 & 1.334 & 1.570 & 1.984 \\
\bWae$\frac{1}{W}$ & 1.000 & 1.004 & 1.007 & 1.019 & 1.035 & 1.090 & 1.298 & 1.516 & 1.527 & 1.000 & 1.001 & 1.005 & 1.033 & 1.069 & 1.144 & 1.283 & 1.586 & 2.025 & 1.000 & 1.001 & 1.005 & 1.031 & 1.064 & 1.131 & 1.284 & 1.532 & 1.940 \\
\bWae$1$ & 1.000 & 1.004 & 1.008 & 1.015 & 1.022 & 1.049 & 1.191 & 1.379 & 1.565 & 1.000 & 1.002 & 1.006 & 1.017 & 1.028 & 1.067 & 1.196 & 1.331 & 1.755 & 1.000 & 1.003 & 1.004 & 1.012 & 1.021 & 1.050 & 1.173 & 1.345 & 1.601 \\
\bWae$W$ & 1.000 & 1.004 & 1.008 & \textbf{1.010} & \textbf{1.013} & \textbf{1.022} & \textbf{1.044} & \textbf{1.078} & \textbf{1.105} & 1.000 & 1.003 & 1.004 & \textbf{1.005} & \textbf{1.007} & \textbf{1.013} & \textbf{1.022} & \textbf{1.037} & \textbf{1.072} & 1.000 & 1.002 & 1.002 & \textbf{1.003} & \textbf{1.004} & \textbf{1.004} & \textbf{1.008} & \textbf{1.015} & \textbf{1.018} \\
\bWae$\lambda^*$ & 1.000 & 1.004 & 1.005 & 1.039 & 1.087 & 1.189 & 1.333 & 1.474 & 1.527 & 1.000 & 1.003 & 1.006 & 1.033 & 1.076 & 1.143 & 1.318 & 1.664 & 2.026 & 1.000 & 1.002 & 1.002 & 1.030 & 1.064 & 1.132 & 1.330 & 1.576 & 1.834 \\

\bottomrule
\end{tabular}
\end{adjustbox}
\end{table}


\FloatBarrier

\section*{Appendix D: Heavy Sliding Tile Puzzle}
Following are tables for node expansions, solution quality, and runtime for heavy STP, i.e., where the cost of moving a tile is its number. The heuristic is heavy Manhattan Distance, which is the same as regular MD, but accounts for the different weighting of moves.

\begin{table}[ht]
\caption{Average number of node expansions on Heavy STP domain.}
\centering
\begin{tabular}{lrrrrrrrrr}
\toprule
\multicolumn{1}{c}{Algorithm} & \multicolumn{1}{c}{1} & \multicolumn{1}{c}{1.1} & \multicolumn{1}{c}{1.2} & \multicolumn{1}{c}{1.5} & \multicolumn{1}{c}{1.7} & \multicolumn{1}{c}{2} & \multicolumn{1}{c}{3} & \multicolumn{1}{c}{5} & \multicolumn{1}{c}{10} \\
\midrule
\wastar & 25M  & 8M  & 3M  & 336K  & 212K  & 115K  & 58K  & 53K  & 45K \\
\bwa & 41M  & 6M  & 1M  & \textbf{174K}  & \textbf{112K}  & 73K  & \textbf{35K}  & \textbf{15K}  & \textbf{9K} \\
\wmm & 39M  & 7M  & 2M  & 566K  & 689K  & 1M  & N/A  & 13M  & 159K \\
\WBSstar & 25M  & 4M  & 964K  & 179K  & 108K  & 74K  & 36K  & 16K  & \textbf{9K} \\
\bWae$\frac{1}{W^2}$ & \textbf{5M}  & \textbf{2M}  & 836K  & 286K  & 149K  & 88K  & 42K  & 17K  & \textbf{9K} \\
\bWae$\frac{1}{W}$ & \textbf{5M}  & \textbf{2M}  & 899K  & 421K  & 232K  & 118K  & 46K  & 19K  & 10K \\
\bWae$1$ & \textbf{5M}  & \textbf{2M}  & 973K  & 622K  & 463K  & 331K  & 81K  & 43K  & 15K \\
\bWae$W$ &  \textbf{5M}  & \textbf{2M}  & 1M  & 949K  & 984K  & 1M  & 1M  & 2M  & 2M \\
\bWae$\lambda^*$ & \textbf{5M} & \textbf{2M} & \textbf{740K} & 213K & 123K & \textbf{72K} & 36K & 19K & \textbf{9K} \\

\bottomrule
\end{tabular}
\end{table}

\begin{table}[ht]
\caption{Average solution quality on Heavy STP domain.}
\centering
\begin{tabular}{lrrrrrrrrr}
\toprule
\multicolumn{1}{c}{Algorithm} & \multicolumn{1}{c}{1} & \multicolumn{1}{c}{1.1} & \multicolumn{1}{c}{1.2} & \multicolumn{1}{c}{1.5} & \multicolumn{1}{c}{1.7} & \multicolumn{1}{c}{2} & \multicolumn{1}{c}{3} & \multicolumn{1}{c}{5} & \multicolumn{1}{c}{10} \\
\midrule
\wastar & 1.000 & 1.007 & 1.020 & 1.095 & 1.157 & 1.259 & 1.532 & 1.878 & 2.305 \\
\bwa & 1.000 & \textbf{1.006} & 1.024 & 1.106 & 1.164 & 1.261 & 1.517 & 1.892 & 2.299 \\
\wmm & 1.000 & 1.008 & 1.031 & 1.126 & 1.205 & 1.335 & N/A & 1.463 & 1.901 \\
\WBSstar & 1.000 & 1.008 & 1.029 & 1.107 & 1.166 & 1.266 & 1.517 & 1.890 & 2.300 \\
\bWae$\frac{1}{W^2}$ & 1.000 & \textbf{1.006} & 1.026 & 1.086 & 1.118 & 1.201 & 1.479 & 1.868 & 2.294 \\
\bWae$\frac{1}{W}$ & 1.000 & 1.008 & 1.025 & 1.064 & 1.099 & 1.164 & 1.404 & 1.775 & 2.241 \\
\bWae$1$ & 1.000 & 1.009 & \textbf{1.019} & 1.048 & 1.064 & 1.092 & 1.258 & 1.535 & 1.910 \\
\bWae$W$ & 1.000 & 1.010 & 1.022 & \textbf{1.037} & \textbf{1.044} & \textbf{1.057} & \textbf{1.093} & \textbf{1.149} & \textbf{1.202} \\
\bWae$\lambda^*$ & 1.000 & 1.008 & 1.024 & 1.094 & 1.160 & 1.256 & 1.512 & 1.811 & 2.300 \\

\bottomrule
\end{tabular}
\end{table}

\begin{table}[ht]
\caption{Average runtime in seconds on Heavy STP domain.}
\centering
\begin{tabular}{lrrrrrrrrr}
\toprule
\multicolumn{1}{c}{Algorithm} & \multicolumn{1}{c}{1} & \multicolumn{1}{c}{1.1} & \multicolumn{1}{c}{1.2} & \multicolumn{1}{c}{1.5} & \multicolumn{1}{c}{1.7} & \multicolumn{1}{c}{2} & \multicolumn{1}{c}{3} & \multicolumn{1}{c}{5} & \multicolumn{1}{c}{10} \\
\midrule
\wastar & 243 & 46.40 & 8.98 & 0.91 & 0.59 & 0.26 & 0.12 & 0.11 & 0.09 \\
\bwa & 200 & 31.97 & 4.33 & \textbf{0.50} & \textbf{0.31} & 0.19 & 0.08 & \textbf{0.03} & \textbf{0.02} \\
\wmm & 5,977 & 831 & 317 & 18.57 & 13.68 & 22.94 & N/A & 104 & 0.57 \\
\WBSstar & 95.62 & 15.02 & \textbf{3.04} & 0.60 & 0.37 & 0.23 & 0.09 & \textbf{0.03} & \textbf{0.02} \\
\bWae$\frac{1}{W^2}$ & 47.23 & 6.68 & 3.37 & 0.95 & 0.45 & 0.24 & 0.10 & \textbf{0.03} & \textbf{0.02} \\
\bWae$\frac{1}{W}$ & \textbf{21.29} & 7.01 & 3.85 & 1.52 & 1.15 & 0.33 & 0.12 & 0.04 & \textbf{0.02} \\
\bWae$1$ & 57.58 & 6.54 & 3.95 & 2.35 & 1.81 & 1.32 & 0.23 & 0.11 & 0.03 \\
\bWae$W$ & 54.78 & 7.36 & 4.57 & 3.87 & 4.03 & 5.23 & 5.49 & 9.04 & 11.14 \\
\bWae$\lambda^*$ & 22.24 & \textbf{5.37} & 3.10 & 0.81 & \textbf{0.31} & \textbf{0.17} & \textbf{0.07} & 0.04 & \textbf{0.02} \\

\bottomrule
\end{tabular}
\end{table}


\FloatBarrier

\clearpage
\section*{Appendix E: Sliding Tile Puzzle with MD-4}
Following are tables for node expansions, solution quality, and runtime for STP which utilizes Manhattan Distance without the 4 center tiles (MD-4).

\begin{table}[ht]
\caption{Average solution quality on STP domain with MD-4.}
\centering
\begin{tabular}{lrrrrrrrrr}
\toprule
\multicolumn{1}{c}{Algorithm} & \multicolumn{1}{c}{1} & \multicolumn{1}{c}{1.1} & \multicolumn{1}{c}{1.2} & \multicolumn{1}{c}{1.5} & \multicolumn{1}{c}{1.7} & \multicolumn{1}{c}{2} & \multicolumn{1}{c}{3} & \multicolumn{1}{c}{5} & \multicolumn{1}{c}{10} \\
\midrule
\wastar & N/A & N/A & N/A & 12M & 3M & 617K & 46K & \textbf{10K} & \textbf{2K} \\
\bwa & N/A & N/A & N/A & 5M & 743K & 263K & \textbf{44K} & 14K & 4K \\
\wmm & N/A & N/A & N/A & 3M & 1M & 621K & 4M & N/A & 2M \\
\WBSstar & N/A & N/A & N/A & 3M & 676K & \textbf{252K} & 45K & 14K & 4K \\
\bWae$\frac{1}{W^2}$ & \textbf{9M} & 5M & \textbf{3M} & 796K & \textbf{493K} & 322K & 66K & 16K & 4K \\
\bWae$\frac{1}{W}$ & \textbf{9M} & \textbf{4M} & \textbf{3M} & 1M & 678K & 575K & 121K & 22K & 5K \\
\bWae$1$ & \textbf{9M} & \textbf{4M} & \textbf{3M} & 2M & 1M & 849K & 334K & 62K & 13K \\
\bWae$W$ & \textbf{9M} & \textbf{4M} & \textbf{3M} & 2M & 2M & 2M & 3M & 4M & 4M \\
\bWae$\lambda^*$ & N/A & N/A & \textbf{3M} & \textbf{776K} & 536K & 342K & 52K & 15K & 5K \\

\bottomrule
\end{tabular}
\end{table}

\begin{table}[ht]
\caption{Average solution quality on STP domain with MD.}
\centering
\begin{tabular}{lrrrrrrrrr}
\toprule
\multicolumn{1}{c}{Algorithm} & \multicolumn{1}{c}{1} & \multicolumn{1}{c}{1.1} & \multicolumn{1}{c}{1.2} & \multicolumn{1}{c}{1.5} & \multicolumn{1}{c}{1.7} & \multicolumn{1}{c}{2} & \multicolumn{1}{c}{3} & \multicolumn{1}{c}{5} & \multicolumn{1}{c}{10} \\
\midrule
\wastar & N/A & N/A & N/A & 1.014 & 1.038 & 1.101 & 1.325 & 1.639 & 2.017 \\
\bwa & N/A & N/A & N/A & 1.014 & 1.046 & 1.116 & 1.348 & 1.674 & 2.019 \\
\wmm & N/A & N/A & N/A & 1.015 & 1.053 & 1.142 & 1.456 & N/A & 1.593 \\
\WBSstar & N/A & N/A & N/A & 1.017 & 1.048 & 1.122 & 1.349 & 1.674 & 2.015 \\
\bWae$\frac{1}{W^2}$ & 1.000 & \textbf{1.000} & \textbf{1.000} & 1.028 & 1.049 & 1.092 & 1.291 & 1.625 & 2.005 \\
\bWae$\frac{1}{W}$ & 1.000 & \textbf{1.000} & 1.004 & 1.023 & 1.041 & 1.073 & 1.230 & 1.565 & 1.983 \\
\bWae$1$ & 1.000 & 1.001 & 1.007 & 1.015 & 1.023 & 1.041 & 1.120 & 1.331 & 1.664 \\
\bWae$W$ & 1.000 & 1.003 & 1.006 & \textbf{1.008} & \textbf{1.009} & \textbf{1.015} & \textbf{1.030} & \textbf{1.046} & \textbf{1.064} \\
\bWae$\lambda^*$ & N/A & N/A & 1.005 & 1.030 & 1.048 & 1.104 & 1.311 & 1.624 & 2.001 \\

\bottomrule
\end{tabular}
\end{table}

\begin{table}[ht]
\caption{Average runtime in seconds on STP domain with MD-4 (in seconds).}
\centering
\begin{tabular}{lrrrrrrrrr}
\toprule
\multicolumn{1}{c}{Algorithm} & \multicolumn{1}{c}{1} & \multicolumn{1}{c}{1.1} & \multicolumn{1}{c}{1.2} & \multicolumn{1}{c}{1.5} & \multicolumn{1}{c}{1.7} & \multicolumn{1}{c}{2} & \multicolumn{1}{c}{3} & \multicolumn{1}{c}{5} & \multicolumn{1}{c}{10} \\
\midrule
\wastar & N/A & N/A & N/A & 131.24 & 9.56 & 1.64 & \textbf{0.10} & \textbf{0.02} & \textbf{0.01} \\
\bwa & N/A & N/A & N/A & 19.50 & 2.57 & 0.80 & 0.11 & 0.03 & \textbf{0.01} \\
\wmm & N/A & N/A & N/A & 12.14 & 3.59 & 1.91 & 13.24 & N/A & 6.41 \\
\WBSstar & N/A & N/A & N/A & 11.68 & 2.28 & \textbf{0.78} & 0.13 & 0.03 & \textbf{0.01} \\
\bWae$\frac{1}{W^2}$ & \textbf{40.25} & 25.40 & 16.95 & 3.17 & \textbf{1.69} & 1.07 & 0.18 & 0.04 & \textbf{0.01} \\
\bWae$\frac{1}{W}$ & 45.19 & 20.58 & 14.42 & 4.12 & 2.52 & 2.05 & 0.36 & 0.05 & \textbf{0.01} \\
\bWae$1$ & 46.08 & \textbf{17.47} & 15.16 & 6.45 & 5.18 & 3.13 & 1.14 & 0.17 & 0.03 \\
\bWae$W$ & 45.30 & 17.59 & 17.69 & 10.43 & 9.55 & 9.03 & 11.93 & 16.06 & 17.68 \\
\bWae$\lambda^*$ & N/A & N/A & \textbf{9.91} & \textbf{2.78} & 1.89 & 1.17 & 0.11 & 0.03 & \textbf{0.01} \\
\bottomrule
\end{tabular}
\end{table}
\newpage
\FloatBarrier
\section{Appendix F: Tuned values of $\lambda$}

In this section, we report the tuned \(\lambda\) values, denoted \(\lambda^*\), obtained after 50 trials of optimization. These values do not necessarily yield the optimal results, but reflect the best performance found within the trial budget. Table~\ref{tab:lambda-values} lists the tuned $\lambda$ values used in \bWae$\lambda^*$ for each domain and heuristic.
In some cases, \(\lambda^*\) may appear unexpectedly high or low, typically in domains where performance is relatively insensitive to \(W\). For example, in Pancake with GAP-2 and \(W = 10\), the tuned value \(\lambda^* = 2.03\) is relatively high. However, Table~4 shows that this choice led to 201 node expansions, compared to 180 for both \(\lambda = \frac{1}{W} = 0.1\) and \(\lambda = \frac{1}{W^2} = 0.01\), indicating that \(\lambda^*\) was slightly suboptimal, though the performance differences are minor. Despite such anomalies, we observe an overall trend: \(\lambda^*\) tends to decrease as \(W\) increases and as the heuristic becomes weaker.

\begin{table}[tbh]
\centering
\caption{Tuned $\lambda$ values used in \bWae$\lambda^*$}
\label{tab:lambda-values}
\begin{tabular}{llccccccccc}
\toprule
Domain & Heuristic & 1 & 1.1 & 1.2 & 1.5 & 1.7 & 2 & 3 & 5 & 10 \\
\midrule
STP & MD & 0.9366 & 0.7220 & 0.3949 & 0.0296 & 0.0311 & 0.1401 & 0.0332 & 0.0296 & 0.0772 \\
STP & Heavy MD & 0.9877 & 0.9211 & 0.4815 & 0.1564 & 0.0097 & 0.0085 & 0.0409 & 0.1280 & 0.0186 \\
STP & MD-4 & N/A & N/A & 0.8489 & 0.4945 & 0.4360 & 0.0816 & 0.0254 & 0.0224 & 0.0616
 \\
Mazes & Octile & 0.9995 & 0.9930 & 0.9827 & 0.9513 & 0.9312 & 0.8953 & 0.7732 & 0.4026 & 0.0008 \\
DAO & Octile & 0.0005 & 0.0019 & 0.0005 & 0.0004 & 0.0002 & 0.0018 & 0.0004 & 0.0029 & 0.0012 \\
\midrule
Pancake & GAP & 0.2025 & 0.6162 & 0.0233 & 0.0376 & 0.1567 & 0.2550 & 0.1521 & 1.3286 & 1.5546 \\
Pancake & GAP-1 & 0.9797 & 0.8970 & 0.5836 & 0.4962 & 0.3863 & 0.4918 & 0.3210 & 0.2358 & 0.0564 \\
Pancake & GAP-2 & 0.8535 & 1.0954 & 0.7973 & 0.3249 & 0.1887 & 0.2383 & 0.4148 & 0.3229 & 2.0255 \\
\midrule
TOH & PDB (10+2) & 0.9933 & 0.9992 & 0.6956 & 0.0390 & 0.0412 & 0.0611 & 0.1104 & 0.2707 & 0.1908 \\
TOH & PDB (8+4) & 0.9906 & 1.0982 & 0.9314 & 0.6608 & 0.3981 & 0.1650 & 0.0345 & 0.0156 & 0.0397 \\
TOH & PDB (6+6) & 0.9878 & 1.0969 & 1.1849 & 0.6977 & 0.5985 & 0.3313 & 0.1547 & 0.0355 & 0.3135 \\
\bottomrule
\end{tabular}
\end{table}

\end{document}